\documentclass{article}
\usepackage{graphicx} 
\usepackage{enumitem}
\usepackage{todonotes}
\usepackage{hyperref}
\usepackage{csquotes}
\usepackage{float}
\usepackage{url}
\usepackage{multicol, multirow,siunitx}
\usepackage{amsfonts,amsmath,amsthm}
\usepackage{booktabs}
\usepackage[main, final]{neurips_2026}
\usepackage{bm}
\usepackage{pgfplots}
\usepgfplotslibrary{fillbetween}
\pgfplotsset{compat=1.18}
\usepackage{tikz}
\usepackage{physics}
\usepackage{mathtools}
\usetikzlibrary{shapes.geometric, arrows}
\theoremstyle{definition}
\newtheorem{definition}{Definition}[section]
\newtheorem{assumption}{Assumption}
\theoremstyle{plain}

\newtheorem{proposition}{Proposition}
\newtheorem{lemma}{Lemma}
\theoremstyle{remark}
\newtheorem{remark}{Remark}

\title{Stability and Generalization in Looped Transformers}
\author{Asher Labovich \\
  Brown University, Liquid AI \\
  \texttt{asherlabovich@gmail.com}}

\begin{document}

\maketitle

\begin{abstract}
Looped transformers promise test-time compute scaling by spending more iterations on harder problems, but it remains unclear which architectural choices let them extrapolate to harder problems at test time rather than memorize training-specific solutions. We introduce a fixed-point based framework for analyzing looped architectures along three axes of stability -- local attractivity, input-dependence, and geometry -- and use it to characterize when fixed-point iteration yields meaningful predictions. Theoretically, we prove that looped networks without recall have countable fixed points and cannot achieve strong input-dependence at any spectral regime, while recall with outer normalization reliably produces a regime in which fixed points are simultaneously locally-attractive, locally smooth in the input, and supported by stable backpropagation. Empirically, we train small looped transformers on chess, sudoku, and prefix-sums and find that downstream performance tracks the framework's predictions across tasks and architectural configurations. We additionally introduce internal recall, a novel recall placement variant allowing an identity residual path, and exploit its contrasting fixed-point behavior to test our framework's predictions in a controlled setting. 
\end{abstract}

\section{Introduction}
Recent progress in reasoning with large language models has come largely from chain-of-thought (CoT) methods, in which models produce a hidden scratchpad of tokens before responding \citep{wei2023chainofthoughtpromptingelicitsreasoning, Guo_2025}. CoT has several structural limitations: it requires autoregressive decoding that scales poorly in latency and energy, it forces intermediate computation to reside in discrete tokens, and its reasoning depth is bounded by the token budget the model was trained on. Looped transformers offer an alternative path to test-time compute scaling by instead training a single weight-tied network whose iteration count can, in principle, be scaled with problem difficulty. Recent work has shown that such models can match or exceed much larger fixed-depth transformers on reasoning benchmarks \citep{geiping_scaling_2025, wang_hierarchical_2025, jolicoeur-martineau_less_2025}, and that weight-tying itself induces useful inductive biases towards algorithmic reasoning \citep{saunshi_inductive_2024, saunshi_reasoning_2025}. Beyond matching fixed-depth performance with few parameters, looped transformers offer a capability unavailable to fixed-depth models: a model that has learned a stable algorithm may be able to extrapolate beyond its training iteration depth, solving problems harder than any it was trained on simply by iterating more.

Whether looped transformers actually achieve this extrapolation remains unclear. Empirical work has converged on two architectural ingredients -- recall (conditioning each iterate on the original input) and outer normalization -- as apparently necessary for stable looped computation \citep{bansal_end--end_2022, anil_path_2022, geiping_scaling_2025}. However, neither is theoretically justified, and generalization results across tasks and scales remain inconsistent \citep{yang2024loopedtransformersbetterlearning}. It is not obvious why recall should be \textit{necessary} rather than merely helpful, why outer normalization (typically disliked in fixed-depth transformers due to gradient instability \citep{xiong2020layernormalizationtransformerarchitecture}) becomes beneficial in looped settings, or how these ingredients interact to enable stable computation.

This paper\footnote{Code available at \url{https://github.com/ashlab11/generalization}} provides a unified account of why these choices matter, framed around a fixed-point analysis of the looped architecture. We argue that a looped model can be trusted to run for arbitrarily many iterations without degrading (avoiding the problem known as \enquote{overthinking}) only if it has fixed points which are locally attractive, input-dependent, and geometrically robust in parameter space -- three properties we call the \enquote{axes of stability}. We show that each architectural choice -- recall existence, recall placement, and outer normalization -- affects a distinct subset of the axes in ways that directly predict downstream performance.

Our key theoretical result is that recall combined with outer normalization reliably yields a regime in which all three axes are simultaneously satisfied. We validate this framework empirically by training single-layer looped transformers on chess, sudoku, and prefix-sums across a grid of normalization and recall choices, and find that downstream performance -- both on the training distribution and harder out of distribution (OOD) problems -- tracks the framework's predictions across tasks and configurations. Over the course of this analysis, we also introduce \textit{internal recall}, a novel placement variant whose narrow stability region without outer normalization provides direct empirical support for the geometry axis of our framework.

\section{Related Work}
In this section, we contextualize our work on looped transformers among the many different architecture, training, and benchmarking choices made in previous work. 

\textbf{Inductive Biases.} Recent work finds that weight-tied architectures exhibit useful inductive biases, trading memorization capacity for reasoning ability \citep{saunshi_inductive_2024, saunshi_reasoning_2025}. In addition, \cite{merrill2025exactexpressivepowertransformers} show that looped transformers with padding can efficiently solve parallelizable problems in a manner unavailable to fixed-depth transformers. However, these results consider models run for a fixed loop count, and none address whether additional test-time iterations can solve problems harder than those seen during training.

\textbf{Extrapolation in Looped Models.} \citet{bansal_end--end_2022} first study extrapolation in CNNs via looped iterations, finding that recall is empirically necessary to avoid overthinking, and introduce the progressive loss mechanism used in this work. \citet{geiping_scaling_2025} scale this to a large transformer, finding outer normalization to be additionally necessary for strong results, though performance plateaus beyond the number of loops seen in training. \citet{yang2024loopedtransformersbetterlearning} corroborate that stable, recall looped transformers can nonetheless fail to generalize to OOD problems. Across all three, recall is consistently identified as necessary but never theoretically justified, and generalization\footnote{In the context of this paper we use \enquote{generalization} to refer specifically to the ability of looped models to extrapolate beyond the number of iterations seen during training to solve OOD problems.} results are mixed.

\textbf{Other Related Work.} Deep Equilibrium Models \citep{bai2019deepequilibriummodels} reframe looped networks as fixed-point finders, directly motivating our stability analysis. Unlike DEQ, which treats convergence as a \textit{design} objective enforced by the solver, we study fixed points as a property of the model itself, analyzing when and why it arises based on architectural choices. Adaptive computation methods \citep{graves2017adaptivecomputationtimerecurrent, dehghani_universal_2018, banino2021pondernetlearningponder} attempt to \textit{reduce} loops for computational efficiency rather than increase them, which makes them complementary to our study. \citet{wang_hierarchical_2025} and \citet{jolicoeur-martineau_less_2025} propose training looped models with \enquote{deep supervision} and find they outperform compute-matched models on reasoning tasks. However, since they only evaluate on in-distribution tasks on training depth, it remains unclear to what extent these gains reflect the inductive biases of weight-tied architectures versus an ability to extrapolate with additional iterations.

\section{Preliminaries}
\subsection{Definitions}\label{sec:defs}

This section provides a short description of looped architectures from a mathematical perspective, and is useful for the propositions proven in the upcoming sections. The framework and most of our theoretical results apply to smooth looped models generally. Transformer-specific structure enters in Section \ref{sec:recall} through our analysis of normalization, and via the token-based transversality condition in Proposition \ref{prop:transversality-weak}.

\begin{definition}[Looped Network]\label{def:looped-network}
    Given parameters $\theta \in \mathbb{R}^p$, embedding dimension $d$, and sequence length $L$, a \emph{looped network} is a discrete dynamical system on the state space $\mathbb{R}^{d \times L}$ defined by the recurrence
    \[
        x_{t+1} = f_\theta\left(x_t, \, \{x_i\}_{i=0}^{t-1}\right) \qquad t = 1, \dots, T-1
    \]
    where $f_\theta$ is some network parameterized by $\theta$ and $x_t \in \mathbb{R}^{d \times L}$ is the hidden state at time t. We call $T$ the \emph{iteration depth}. The network is initialized with $x_1 = f_\theta(e, x_0)$ where $e$ is a free parameter often chosen to be zero, $x_0$, or Gaussian noise. This formulation of looped networks is quite broad; the following definitions narrow this to the two cases studied in this paper.
\end{definition}

\begin{definition}[Autonomous Network]\label{def:autonomous}
    A looped network is \emph{autonomous} if $f_\theta$ depends only on the current state, i.e.,
    \[
        x_{t+1} = f_\theta(x_t).
    \]
    We choose the term \enquote{autonomous} to mimic autonomous systems in differential equations theory, which arise when the rate of change depends only on the current state.
\end{definition}

\begin{definition}[Recall Network]\label{def:recall}
    A looped network is a \emph{recall network} if $f_\theta$ depends on only the current state and the initial input, i.e.,
    \[
        x_{t+1} = f_\theta(x_t, \, x_0).\qquad t = 1, \dots, T-1
    \]
    Recall models are initialized at $x_1$ equivalently to Definition \ref{def:looped-network}.
\end{definition}

A recall network can be thought of as \enquote{middle ground} between autonomous and the more general looped models, taking in a second parameter (unlike autonomous models) but limiting that parameter to only the initial state (unlike general looped models). Figure \ref{fig:autonomous-vs-recall} compares autonomous and recall models visually.

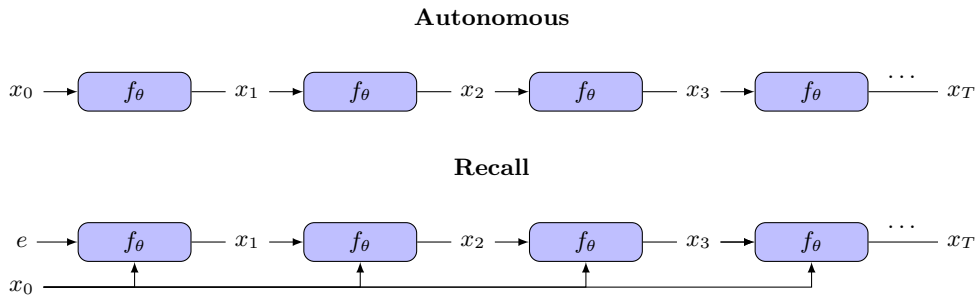
\begin{figure}[h!]
\centering
\begin{tikzpicture}[
    node distance=0.9cm and 0.9cm,
    >=latex,
    box/.style={draw, rounded corners, minimum width=1.5cm, minimum height=0.8cm, align=center, font=\small},
    smallbox/.style={draw, rounded corners, minimum width=1.5cm, minimum height=0.7cm, align=center, font=\small},
    xbox/.style={draw, rounded corners, minimum width=1.5cm, minimum height=0.5cm, align=center, font=\small, 
    fill=gray!50},
    fbox/.style={draw, rounded corners, minimum width=1.5cm, minimum height=0.5cm, align=center, font=\small, 
    fill=blue!25},
    gbox/.style={draw, rounded corners, minimum width=1.5cm, minimum height=0.5cm, align=center, font=\small, 
    fill=teal!50},
    phibox/.style={draw, dotted, rounded corners, minimum width=1.5cm, minimum height=0.5cm, align=center, font=\small},
    lbl/.style={font=\small}
]

\node[lbl] (title1) at (4.75,2.2) {\textbf{Autonomous}};
\foreach \i in {0,1,2,3} { 
    \node[lbl] (x\i) at (\i*3 - 1.5, 1.2) {$x_\i$};
    \node[fbox] (f\i) at (\i*3, 1.2) {$f_\theta$};
}
\foreach \i in {0, 1, 2}{
    \pgfmathtruncatemacro{\j}{\i+1};
    \draw[->] (x\i) -- (f\i);
    \draw[-] (f\i) -- (x\j);
}
\draw[->] (x3) -- (f3);
\node[lbl] (xT) at (11, 1.2) {$x_T$};
\draw[-] (f3) -- (xT);
\node[lbl] (dots) at (10.2, 1.4) {$\hdots$};

\node[lbl] (title1) at (4.75,0.2) {\textbf{Recall}};
\node[lbl] (xr0) at (-1.5, -1.4) {$x_0$};
\foreach \i in {1,2,3} { 
    \node[lbl] (xr\i) at (\i*3 - 1.5, -0.8) {$x_\i$};
}
\foreach \i in {0,1,2,3} { 
    \node[fbox] (fr\i) at (\i*3, -0.8) {$f_\theta$};
}
\foreach \i in {0, 1, 2}{
    \pgfmathtruncatemacro{\j}{\i+1};
    \draw[-] (fr\i) -- (xr\j);
}
\draw[->] (xr0) -- (0, -1.4) -- (fr0.south); 
\foreach \i in {1,2,3} {
    \draw[->] (xr\i) -- (fr\i);
    \draw[->] (xr0) -- (\i*3, -1.4) -- (fr\i.south);
}
\draw[->] (xr3) -- (fr3);
\node[lbl] (xrT) at (11, -0.8) {$x_T$};
\node[lbl] (dots) at (10.2, -0.6) {$\hdots$};
\draw[-] (fr3) -- (xrT);
\node[lbl] (e) at (-1.5, -0.8) {$e$};
\draw[->] (e) -- (fr0);
\end{tikzpicture}
\caption{Comparison of autonomous and recall networks. In the former, the network only depends on the previous state $x_{t-1}$; in the latter, it also depends on the initial state $x_0$.}
\label{fig:autonomous-vs-recall}
\end{figure}

\begin{definition}[Outer Normalization]\label{def:post-norm}
    A looped network uses \emph{outer normalization} if some normalizing function is applied to the output of each iteration: \[
        x_{t+1} = \phi_\theta\left(f_\theta\left(x_t, \{x_i\}_{i=0}^{t-1}\right)\right)
    \]
    where $f_\theta$ may itself contain arbitrary internal normalization.
\end{definition}

\begin{remark}\label{rem:post-norm}
Practitioners have considerable flexibility in their choice of $\phi$; we consider RMSNorm and GRU normalization. When $\phi$ is RMSNorm or LayerNorm, this matches what the standard transformer literature calls \textit{post-norm}, and we use that term in this setting. In fixed-depth transformers, pre-norm -- applying normalization only inside $f_\theta$ -- is more common due to gradient instability problems with post-norm \citep{xiong2020layernormalizationtransformerarchitecture}; we show in Section 5.2 that recall architectures create a stability regime which uniquely benefits from outer normalization. We also note that post-norm constrains each token to a compact set, which we exploit in Section \ref{sec:autonomous}.
\end{remark}

\paragraph{Smoothness Assumption.} Separately from our definitions, we also note that throughout this paper, we assume all models are \textit{smooth}, which can be guaranteed by using smooth activations only\footnote{Many modern transformers use GeLU or SwiGLU activations, both of which are smooth. In addition, RMSNorm/LayerNorm are both smooth as long as the denominator contains an $\epsilon$ term.}.

\subsection{Axes of Stability}\label{sec:axes}
Throughout this paper, we discuss architectural choices that improve or harm the ability of looped transformers to learn algorithms rather than memorize solutions. To make this analysis precise, we first need a vocabulary for the distinct ways such a model can succeed or fail. We emphasize here that our analysis focuses on fixed points of the looped computation. In practice, we want looped models we can trust to run for any number of iterations without overthinking or degrading. Only fixed points offer this guarantee, so we take reaching a correct fixed point as  our criterion for success.

For each task, we separate easy data (matching the training distribution) from hard data (out-of-distribution along the generalization axis of interest -- longer sequences, harder instances). We identify three properties (\enquote{axes}) of a looped model that jointly determine its success on easy and/or hard data: \textbf{local attractivity, input-dependence, and geometry}. We describe each in turn below.

\textbf{Locally attractive.} A looped model is locally attractive if repeated iteration near a fixed point converges to that fixed point rather than diverging or cycling. In practice, this means running the model for more iterations moves it toward a stable answer rather than away from one. Local attractivity is the most basic dynamical requirement for looped computation. Due to floating-point perturbations, function iteration will rarely if ever reach a fixed point if the fixed point is not locally attractive.

\textbf{Input-dependence.} It is trivial to build a model with perfect attraction by ensuring a contraction map and relying on the Banach fixed point theorem. Such a map is, of course, useless as an algorithm -- its output is the same regardless of input. A useful looped model must therefore have fixed points that vary meaningfully with the input. Architectures differ in \textit{how} they achieve this, and as we will see, the mechanism matters as much as whether input-dependence exists at all.

\textbf{Geometry.} Even among models that have parameter regimes of fixed-point attraction and input-dependence, the \textit{shape} of that regime matters. Stability regions that are highly direction-dependent -- narrow slivers rather than round, broad neighborhoods -- force stable configurations into a narrow band of the parameter space, making training outcomes fragile and hyperparameter-sensitive. We refer to this directional bias as anisotropy, and show in Section \ref{sec:recall-no-outer} and Appendix \ref{app:anisotropy} that internal and external recall without outer normalization differ sharply along this axis.

Having defined the axes of stability, we now examine how specific architectural choices -- recall existence, recall placement, and outer normalization -- succeed or fail along each axis, and its effect on downstream performance.

\section{Why Autonomous Networks Fail}\label{sec:autonomous}
Prior work finds that autonomous looped networks often achieve low accuracy and degrade under repeated iteration \citep{bansal_end--end_2022}, with networks lacking input injection producing  \enquote{essentially random or unpredictable} outputs \citep{yang2024loopedtransformersbetterlearning}. In this section, we show this failure is \textit{structural}, not merely an artifact of poor optimization or architecture choice: under mild assumptions, autonomous models exhibit extremely weak input-dependence across all spectral regimes. 

\begin{assumption}\label{assum:transversal}
    Consider function $g(x_t, \theta): \mathbb{R}^{d \times L} \times \mathbb{R}^p \rightarrow \mathbb{R}^{d \times L} \coloneqq f_\theta(x_t) - x_t$ chosen so that $g(x_t, \theta) = 0 \Leftrightarrow x_t$ is a fixed point of $f_\theta$. We assume that at all points where $g(x_t, \theta) = 0$, $\begin{bmatrix}
        \pdv{g(x_t, \theta)}{x_t} & \pdv{g(x_t, \theta)}{\theta}
    \end{bmatrix}$ spans $\mathbb{R}^{d \times L}$ (i.e. g is transversal to \{0\}\footnote{Transversality is defined as follows: a smooth function $f: N \rightarrow M$ is transverse to a manifold T if at all points in $f^{-1}(T)$, the Jacobian of f combined with the tangent space of T generates the tangent space of M at the image point.}).
\end{assumption}

This assumption is quite mild in practice, failing only when \emph{all} $p + dL$ variables conspire to lose rank at a fixed point, which is an extreme coincidence. Examples of the weakness of this assumption can be found in Appendix \ref{app:transversality}. Despite its weakness, this assumption helps us prove powerful results for autonomous networks. 
\begin{proposition}[Restrictions on the Jacobian of $f_\theta$ at fixed points]\label{prop:fixed-jacobian}
    For almost all parameter vectors $\theta \in \mathbb{R}^p$, $J_{f_\theta}(x^*) = \pdv{f_\theta(x^*)}{x^*}$ contains no eigenvalues equal to 1 for any fixed points $x^*$.
\end{proposition}
\begin{proof}
    Since g is smooth and transversal to \{0\} (Assumption \ref{assum:transversal}), we can apply the transversality theorem (\citep[p.~68]{guillemin_differential_2010}) to get that for almost all $\theta$, $g_\theta$ is transversal to 0. By definition of a transversal, we have that $J_{g_\theta}$ together with $T_0\{0\} = \{0\}$ span $\mathbb{R}^{d \times L}$ and thus $J_{g_\theta}$ is non-singular at $g^{-1}_\theta(0)$. However, \begin{equation}\label{eq:jacobianf}
        J_{g_\theta} = \pdv{g_\theta(x)}{x} = \pdv{f_\theta(x) - x}{x} = J_{f_\theta} - I 
    \end{equation}
    \eqref{eq:jacobianf} implies that $J_{f_\theta} = J_{g_\theta} + I$. Since $J_{g_\theta}$ is non-singular, $J_{f_\theta}$ has no eigenvalues of 1 at $g^{-1}_\theta(0)$ (the fixed points of $f_\theta$).
\end{proof}

With this, we can characterize the fixed-point structure of autonomous networks.

\begin{proposition}[Fixed points are a dimension-zero manifold]\label{prop:dim0}
    Let $f_\theta: \mathbb{R}^{d \times L} \rightarrow \mathbb{R}^{d \times L}$ be an autonomous looped network. Let $S = \{x^*: f_\theta(x^*) = x^*\}$ be the set of fixed points of $f_\theta$. Then, for almost all parameterizations $\theta$, \textit{S is a dimension-zero manifold}.  
\end{proposition}
\begin{proof}
    From Proposition \ref{prop:fixed-jacobian}, we know that for almost all $\theta$, $J_{g_\theta} = J_{f_\theta(x) - x}$ is non-singular at points $x^*$ where $g_\theta(x^*) = 0$. Thus, 0 is a regular value and thus $g^{-1}_\theta(0)$ is a manifold of dimension $dL - dL = 0$ \citep[p.~11]{Milnor1965}. Since  $g^{-1}_\theta(0)$ represents the set of all fixed points of $f_\theta$ the fixed points are a manifold of dimension 0 (countable).
\end{proof}

Proposition \ref{prop:dim0} bounds the number of fixed points an autonomous network can have: for almost all parameters, the set is countable, and finite when $x$ is confined to a compact set (as under post-norm; see Remark \ref{rem:post-norm}). This limits the input-dependence of an autonomous model to simple basin selection, where the network simply chooses \textit{which} of its predetermined isolated fixed points to output given an input. Recall networks, which condition each step on $x_0$, are not subject to this constraint. We note that this argument extends naturally to \emph{joint fixed points} — points where the output of a final readout head $w_\theta$ is unchanged by one application of $f_\theta$ even if the hidden state is not -- with an analogous transversality assumption; see Appendix \ref{app:joint-fixed} for details.

The previous proposition argues that autonomous networks have severely limited input-dependence in theory. The following proposition shows that even within this limited regime, the gradient dynamics make basin selection hard to learn.

\begin{proposition}[Input-Gradient Vanishing of Autonomous Looped Networks]\label{prop:instability-autonomous} 
Let $f_\theta$ be an autonomous looped network. Suppose  $x_t$ converges to $x^*$ with $\rho(J_{f_\theta}) < 1$. Then $\dv{x_T}{x_0}$ converges to 0 exponentially as $T \rightarrow \infty$.
\end{proposition}

\begin{proof}
    See Appendix \ref{app:grad-instability} for details.
\end{proof}

\begin{figure}[t!]
\centering
\begin{tikzpicture}

\begin{scope}
  \begin{scope}
    \clip (-2.5,-2.5) rectangle (2.5,2.5);
    \foreach \xi in {-2,-1.5,-1,-0.5,0.5,1,1.5,2} {
      \foreach \yi in {-2,-1.5,-1,-0.5,0.5,1,1.5,2} {
        \pgfmathsetmacro{\ddx}{-0.50*\xi*0.55}
        \pgfmathsetmacro{\ddy}{-0.35*\yi*0.55}
        \draw[->,thin,black!45] (\xi,\yi) -- ++(\ddx,\ddy);
      }
    }
  \end{scope}
  \draw[->] (-2.5,0) -- (2.5,0) node[right,font=\footnotesize]{$x$};
  \draw[->] (0,-2.5) -- (0,2.5) node[above,font=\footnotesize]{$y$};
  \filldraw (0,0) circle (2.5pt)
    node[above right,font=\footnotesize]{$x^*$};
  \node[font=\small] at (0,-3.1)
    {(a) $|\lambda_x|,|\lambda_y|<1$};
\end{scope}

\begin{scope}[xshift=6.5cm]
  \begin{scope}
    \clip (-2.5,-2.5) rectangle (2.5,2.5);
    \foreach \xi in {-2,-1.5,-1,-0.5,0.5,1,1.5,2} {
      \foreach \yi in {-2,-1.5,-1,-0.5,0.5,1,1.5,2} {
        \pgfmathsetmacro{\ddx}{-0.50*\xi*0.55}
        \pgfmathsetmacro{\ddy}{ 0.80*\yi*0.55}
        \draw[->,thin,black!45] (\xi,\yi) -- ++(\ddx,\ddy);
      }
    }
  \end{scope}
  \draw[->,red!75!black,very thick] (-2.5,0) -- (2.5,0)
    node[right,font=\footnotesize,black]{$x$};
  \draw[->] (0,-2.5) -- (0,2.5) node[above,font=\footnotesize]{$y$};
  \node[red!75!black,font=\footnotesize,above] at (1.55,0.05)
    {stable mfd.};
  \filldraw (0,0) circle (2.5pt)
    node[above right,font=\footnotesize]{$x^*$};
  \node[font=\small] at (0,-3.1)
    {(b) $|\lambda_x|<1<|\lambda_y|$};
\end{scope}

\end{tikzpicture}
\caption{Phase portraits for the map $(x,y)\mapsto(\lambda_x x,\,\lambda_y y)$, corresponding to a $2\times2$ matrix with eigenvectors along the coordinate axes. 
\textbf{(a)} Both eigenvalues satisfy $|\lambda|<1$: every trajectory converges to $x^*$. This is \textit{input-gradient vanishing.} 
\textbf{(b)} One eigenvalue exceeds~1 in magnitude: only the stable manifold
(the $x$-axis, shown in red) converges to $x^*$; \textit{all other trajectories diverge}.}
\label{fig:phase-portraits}
\end{figure}
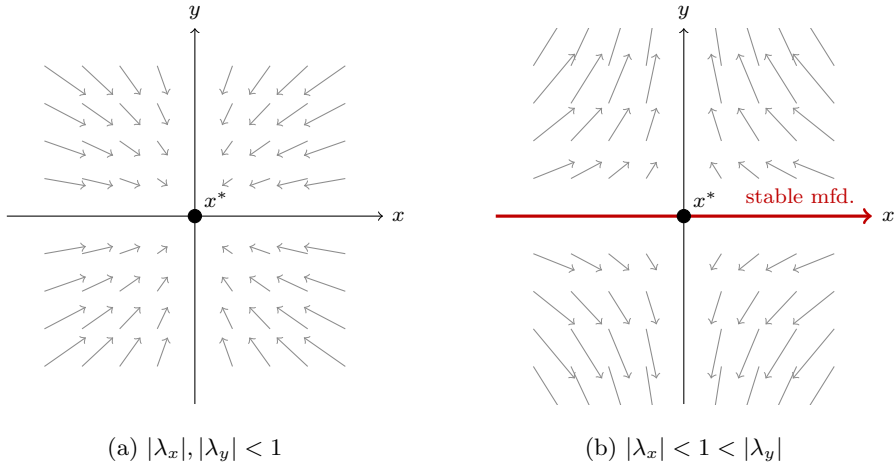

This representation provides a stronger understanding of the dynamics of the gradient at a fixed point of $f_\theta$. In Appendix \ref{app:cases}, we show that the behavior of $\pdv{x_T}{x_0}$ depends strongly on the value of $\rho(J_{f_\theta}(x^*))$, the spectral radius of the Jacobian of $f_\theta$ at the fixed point. In particular, we find three regimes, none of which provides useful training signal. If $\rho(J) < 1$, $x^*$ is locally attractive; thus Proposition \ref{prop:instability-autonomous} holds, and the input-gradient vanishes exponentially with T. If $\rho(J) > 1$, the fixed point is not locally attractive and thus not practically reached with fixed-point iteration in the first place. Finally, if $\rho(J) = 1$, fixed-point iteration faces very slow convergence with possible exploding parameter gradients. Figure \ref{fig:phase-portraits} illustrates the contrast between the cases where $\rho(J) < 1$ versus $\rho(J) > 1$: when all eigenvalues have modulus less than 1, every trajectory converges to $x^*$; when any eigenvalue exceeds 1 in magnitude, there exists a direction which is locally repelling, and trajectories almost always escape the neighborhood.

Together with Proposition \ref{prop:dim0}, this paints a difficult picture for autonomous networks: their expressible outputs are countable, and reaching those fixed points trades off against the quality of the training signal. This helps explain the empirical observations of \citet{bansal_end--end_2022}, where autonomous networks fail to reach fixed points and substantially degrade under repeated iterations.

By defining the \enquote{best} models as those that achieve top hard accuracy \textit{at 500 iterations}, we naturally choose those that overthink the least, avoiding the failure mode noted in \citet{bansal_end--end_2022}\footnote{Many of our runs \textit{did} see this overthinking, but some norm/lr combinations did not.}. But the bottleneck remains: avoiding overt degradation under extra iterations is easier than having locally attractive, input-dependent fixed points. Consistent with this, our autonomous models (Appendix \ref{app:full-results}, \enquote{none} rows) achieve non-trivial val accuracy in some configurations, but still exhibit substantial val-to-hard degradation and never approach the hard accuracy of recall variants on the same task.

Notably, prefix-sums is the only problem on which autonomous networks achieve 0\% accuracy across all configurations. It is also the only problem whose input dimensionality changes between training and testing: the train distribution is 32-bit sequences and the test distribution is 512-bit. Since the set of fixed points depends on the state dimension, basin-selection learned at training dimension largely cannot transfer to test dimension. Sudoku and chess preserve input dimensionality across train and test, so basin-selection can at least partially transfer -- consistent with the non-trivial val and hard accuracies we observe there.

Recall networks aim to fix this by providing additional paths from $x_T \rightarrow x_0$ by adding it as an input at every step. However, simply introducing recall is not sufficient for fixed-point stability -- the \textit{placement} of recall, as well as the presence of outer normalization, critically determines where on the axes of stability a given recall architecture lies. We analyze this in the following section.

\section{Stability and Generalization of Recall Networks}\label{sec:recall}

Although the formal definition of recall is simple -- merely conditioning each iterate on the original input -- practitioners have considerable flexibility in exactly how this conditioning is applied. This flexibility is compounded in models with multiple sublayers, where a nontrivial additional question arises: \textit{where} in the network is the initial state injected? Below, we define two architecturally distinct recall variants, one standard in the literature and one novel, whose fixed-point structures differ in ways that directly determine stability and downstream performance. Our formal definitions assume a \textbf{two-sublayer} model for generality (e.g. Attention and MLP); Figure \ref{fig:internal-vs-external} gives a visual explanation of both variants under such a model. All functions depend on $\theta$ (omitted for brevity), and $\phi_i$ acts as an outer normalization function that may be present or absent depending on architecture choice.

\begin{figure}[h!]
\centering
\resizebox{0.9\linewidth}{!}{%
\begin{tikzpicture}[
    node distance=0.9cm and 0.9cm,
    >=latex,
    box/.style={draw, rounded corners, minimum width=1.5cm, minimum height=0.8cm, align=center, font=\small},
    smallbox/.style={draw, rounded corners, minimum width=1.5cm, minimum height=0.7cm, align=center, font=\small},
    fbox/.style={draw, rounded corners, minimum width=1.5cm, minimum height=0.5cm, align=center, font=\small, 
    fill=blue!25},
    gbox/.style={draw, rounded corners, minimum width=1.5cm, minimum height=0.5cm, align=center, font=\small, 
    fill=teal!50},
    phibox/.style={draw, dotted, rounded corners, minimum width=1.5cm, minimum height=0.5cm, align=center, font=\small},
    lbl/.style={font=\small}
]

\node[lbl] (title1) at (-4.8,3) {\textbf{External Recall}};
\node[lbl] (x1) at (-11,1.2) {$x_t$};
\node[gbox] (g1) at (-9, 1.2) {$g$};
\node[lbl] (x0) at (-11, 0.4) {$x_0$};
\node[circle, draw, inner sep=1pt] (plus1) at (-5.5, 1.2) {$+$};
\node[fbox] (h1) at (-7, 1.8) {$h_1$};
\draw[->] (x1) -- (g1);
\draw[->] (x0) -| (g1);
\draw[-] (g1) -- (plus1);
\draw[->] (g1) -| (h1);
\draw[-] (h1) -- (-5.5, 1.8) -- (plus1.north);
\node[phibox] (phi1) at (-4, 1.2) {$\phi_1$}; 
\draw[-] (plus1.east) -- (phi1.west);
\node[fbox] (h2) at (-2, 1.8) {$h_2$};
\node[circle, draw, inner sep=1pt] (plus2) at (-0.5, 1.2) {$+$};
\draw[-] (phi1) -- (plus2);
\draw[->] (phi1) -- (-2, 1.2) -- (h2.south);
\draw[-] (h2.east) -- (-0.5, 1.8) -- (plus2);
\node[phibox] (phi2) at (1, 1.2) {$\phi_2$}; 
\draw[-] (plus2) -- (phi2);
\draw[->] (phi2) -- (1, 2.25) -- (-11, 2.25) -- (x1);
\node[lbl] (xT) at (-4.8, 2.5) {$\times T$};

\node[lbl] (title2) at (-4.8,0) {\textbf{Internal Recall}};
\node[lbl] (xi1) at (-12,-1.3) {$x_t$};
\node[gbox] (gi1) at (-10, -2.1) {$g$};
\node[lbl] (xi0) at (-12, -2.1) {$x_0$};
\node[circle, draw, inner sep=1pt] (plusi1) at (-6.5, -1.3) {$+$};
\node[fbox] (hi1) at (-8, -2.1) {$h_1$};
\draw[->] (xi1) -- (plusi1);
\draw[->] (xi1) -- (-10, -1.3) -- (gi1);
\draw[->] (xi0) -- (gi1);
\draw[->] (gi1) -- (hi1);
\draw[-] (hi1) -- (-6.5, -2.1) -- (plusi1.south);
\node[phibox] (phii1) at (-5, -1.3) {$\phi_1$}; 
\draw[-] (plusi1.east) -- (phii1.west);
\node[lbl] (xT) at (-4.8, -0.5) {$\times T$};

\node[gbox] (gi2) at (-3, -2.1) {$g$};
\node[lbl] (xi01) at (-5, -2.1) {$x_0$};
\node[circle, draw, inner sep=1pt] (plusi2) at (0.5, -1.3) {$+$};
\node[fbox] (hi2) at (-1, -2.1) {$h_2$};
\draw[->] (phii1) -- (plusi2);
\draw[->] (phii1) -- (-3, -1.3) -- (gi2);
\draw[->] (xi01) -- (gi2);
\draw[->] (gi2) -- (hi2);
\draw[-] (hi2) -- (0.5, -2.1) -- (plusi2.south);
\node[phibox] (phii2) at (2, -1.3) {$\phi_2$}; 
\draw[-] (plusi2) -- (phii2);
\draw[->] (phii2) -- (2, -0.8) -| (xi1); 
\end{tikzpicture}%
}
\caption{Comparison of two-layer external and internal recall architectures. In external recall, the recalled representation replaces the residual stream; in internal recall, the current state remains on the residual path and recall only affects the update.}
\label{fig:internal-vs-external}
\end{figure}
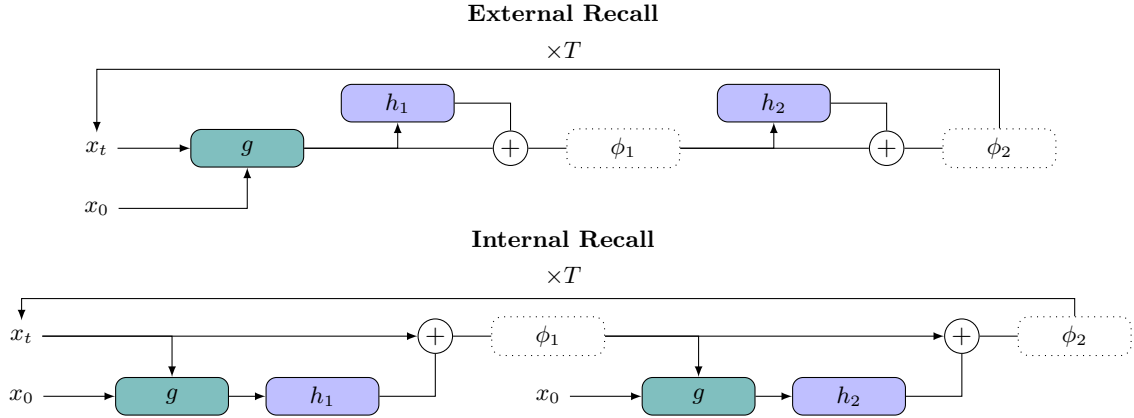

\begin{definition}[External Recall]
    An \textit{external} recall network is one such that for $t > 1$, \begin{align*}
        x_{t+1} &= \phi_2(z_t + h_2(z_t)) \\
        z_t &= \phi_1(g(x_t, x_0) + h_1(g(x_t, x_0)))
    \end{align*}
    
\end{definition}

External recall is the most common architectural choice in the literature, as in \cite{bansal_end--end_2022, geiping_scaling_2025, wang_hierarchical_2025, jolicoeur-martineau_less_2025}. Common choices for $g$ are pointwise addition \citep{wang_hierarchical_2025, jolicoeur-martineau_less_2025} or a linear combination with trainable matrices $W_1, W_2$ \citep{bansal_end--end_2022, geiping_scaling_2025}. We note that recall only enters in the \textit{first} sublayer, with all later sublayers acting as pseudo-autonomous layers. 

\begin{definition}[Internal Recall]
    An \textit{internal} recall network is one such that for $t > 1$, \begin{align*}
        x_{t+1} = \phi_2(z_t + h_2(g(z_t, x_0))) \\
        z_t = \phi_1(x_t + h_1(g(x_t, x_0)))
    \end{align*}
\end{definition}

In internal recall, the recall function $g$ never enters into the residual directly, only ever affecting the update to the residual. For a two-sublayer model, internal and external recall use the same parameter count and nearly identical compute (internal applies g once more per iteration, but g is generally inexpensive). This allows for controlled comparisons between the two architectures in the remainder of Section \ref{sec:recall}.

\subsection{Recall Without Outer Normalization}\label{sec:recall-no-outer}

\begin{figure}[h!]
\centering
\begin{tikzpicture}
\begin{axis}[
    width=10.5cm,height=7cm,
    xmin=-10,xmax=10,
    ymin=-6,ymax=6,
    axis lines=middle,
    axis line style={gray!55, very thin},
    tick style={gray!55, very thin},
    ticklabel style={font=\small},
    xtick={-10,-5,0,5,10},
    xlabel={$J_g$}, 
    ylabel={$J_h$},
    ytick={-6,-4,-2,0,2,4,6},
    samples=400,
    clip=true,
    title={Stability Regions of Simple Model},
    title style={yshift=8pt}
]

\addplot[name path=rp,draw=none,domain=0.08:10] {-2/x};
\addplot[name path=rx,draw=none,domain=0.08:10] {0};
\addplot[red,fill=red,fill opacity=.35,draw=none] fill between[of=rp and rx];

\addplot[name path=rn,draw=none,domain=-10:-0.08] {-2/x};
\addplot[name path=rz,draw=none,domain=-10:-0.08] {0};
\addplot[red,fill=red,fill opacity=.35,draw=none] fill between[of=rn and rz];

\addplot[red,dashed,thick,domain=-10:-0.08] {-2/x};
\addplot[red,dashed,thick,domain=0.08:10] {-2/x};

\addplot[name path=bp1,draw=none,domain=0.08:10] {1/x-1};
\addplot[name path=bp2,draw=none,domain=0.08:10] {-1/x-1};
\addplot[blue,fill=blue,fill opacity=.35,draw=none] fill between[of=bp1 and bp2];

\addplot[name path=bn1,draw=none,domain=-10:-0.08] {1/x-1};
\addplot[name path=bn2,draw=none,domain=-10:-0.08] {-1/x-1};
\addplot[blue,fill=blue,fill opacity=.35,draw=none] fill between[of=bn1 and bn2];

\addplot[blue,dashed,thick,domain=-10:-0.08] {1/x-1};
\addplot[blue,dashed,thick,domain=0.08:10] {1/x-1};
\addplot[blue,dashed,thick,domain=-10:-0.08] {-1/x-1};
\addplot[blue,dashed,thick,domain=0.08:10] {-1/x-1};

\node[
    anchor=north east,
    draw=gray!30,
    fill=white,
    rounded corners=2pt,
    inner sep=3pt,
    font=\small
] at (rel axis cs:0.98,0.97) {%
\begin{tabular}{@{}c@{\hspace{4pt}}l@{}}
\raisebox{0.3ex}{\fcolorbox{red}{red!35}{\rule{0pt}{0.6ex}\rule{0.6ex}{0pt}}} &
$|1+\lambda_{J_g}\lambda_{J_h}|<1$ \\[2pt]
\raisebox{0.3ex}{\fcolorbox{blue}{blue!35}{\rule{0pt}{0.6ex}\rule{0.6ex}{0pt}}} &
$|(1+\lambda_{J_h})\lambda_{J_g}|<1$
\end{tabular}%
};

\end{axis}
\end{tikzpicture}
\caption{Stability regions for a simplistic (single-layer, equivalent eigenvectors) model of external (blue) and internal (red) recall models. In this figure, we shade the regions for which the two eigenvalues produce a stable fixed point.}
\label{fig:stability_regions}
\end{figure}
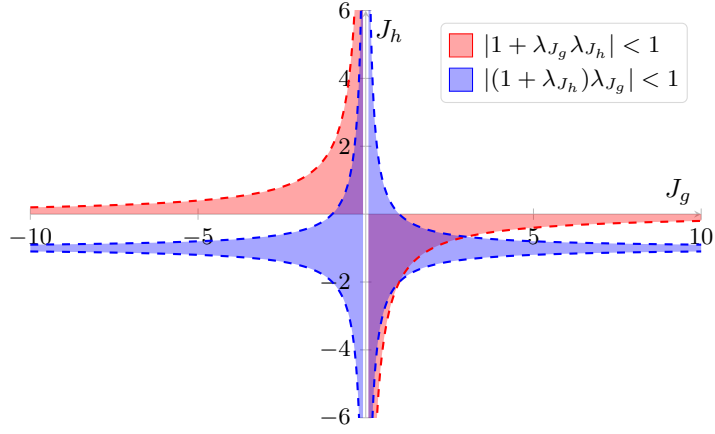

In Proposition \ref{prop:ext-int}, we discuss the stability of these models in the scenario where outer normalization is not applied; i.e. $\phi_i = I$.

\begin{proposition}[Local attractivity conditions for a recall model without outer normalization]\label{prop:ext-int}
    Let $x^*$ be a fixed point of a recall model for a fixed input $x_0$. \\
    \textbf{External Recall:} 
    Define \[
    M_{\text{ext}} = \left(I + \dv{h_2(z^*)}{z^*}\right) \left(I + \dv{h_1(g(x^*, x_0))}{g(x^*, x_0)} \right)\pdv{g(x^*, x_0)}{x^*}
    \]
    Then, $\rho(M_{\text{ext}}) < 1$ is a sufficient condition for $x^*$ to be a locally attracting fixed point, while $\rho(M_{\text{ext}}) \leq 1$ is a necessary one. \\
    \textbf{Internal Recall:} 
    Define \[
    M_{\text{int}} = \left(I + \dv{h_2(g(z^*, x_0))}{g(z^*, x_0)} \pdv{g(z^*, x_0)}{z^*}\right) \left(I + \dv{h_1(g(x^*, x_0))}{g(x^*, x_0)} \pdv{g(x^*, x_0)}{x^*}\right)
    \]Then, $\rho(M_{\text{int}}) < 1$ is a sufficient condition for $x^*$ to be a locally attracting fixed point, while $\rho(M_{\text{int}}) \leq 1$ is a necessary one. \\ 
\end{proposition}

\begin{proof}
See Appendix \ref{app:stability-no-outer}. 
\end{proof}

We note that while recall models do indeed add an additional path from $x_t \rightarrow x_0$, improving input-dependence, without outer normalization they are harder to stabilize than it may first appear. In particular, the \textit{geometry} of the locally-attractive, input-dependent fixed-point regime changes drastically depending on whether the chosen architecture uses internal or external recall. Figure \ref{fig:stability_regions} visualizes this for a simplified single-layer model in which the eigenvectors of $J_h = \dv{h(g(x^*, x_0))}{g(x^*, x_0)}$ and $J_g = \pdv{g(x^*, x_0)}{x^*}$ align, so that the stability condition reduces to a simple multiplication of scalar values.

\begin{figure}[!htb]
    \centering
    \begin{minipage}{0.5\textwidth}
        \centering
        \includegraphics[width=\linewidth]{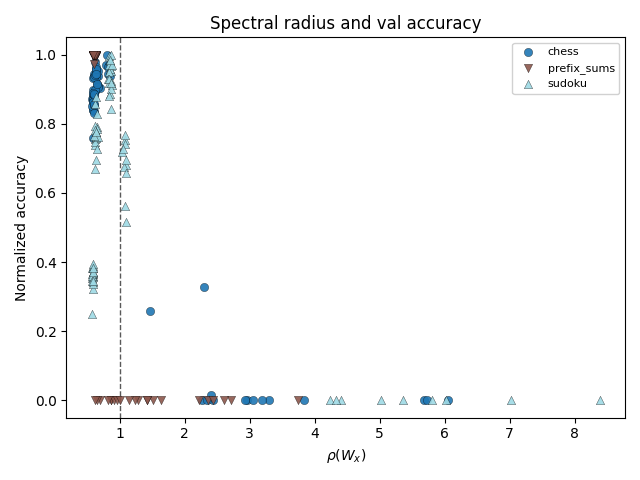}
    \end{minipage}\hfill
    \begin{minipage}{0.5\textwidth}
        \centering
        \includegraphics[width=\linewidth]{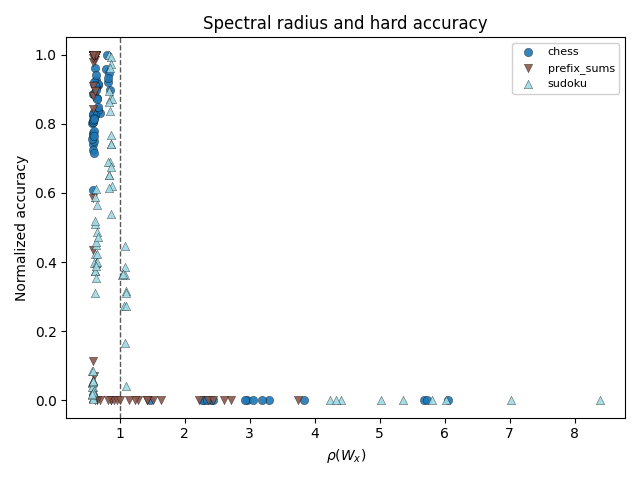}
    \end{minipage}
    \caption{External unnormalized recall validation and hard (normalized by task) accuracy mapped against $\rho(\pdv{g(x^*, x_0)}{x^*})$. External recall achieves substantially worse performance on all tasks when $\rho(\pdv{g(x^*, x_0)}{x^*}) > 1$, as suggested by Proposition \ref{prop:ext-int}.}
    \label{fig:spectral}
\end{figure}

In the case where $\pdv{g(x^*, x_0)}{x^*}$ is strongly contractive, $M_{\textbf{ext}}$ allows substantial flexibility on the other terms while retaining stability. The same is not true for internal recall, whose identity shifts require a careful balancing of $\pdv{h_2(g(z^*, x_0))}{g(z^*, x_0)}$ and $\pdv{h_1(g(x^*, x_0))}{g(x^*, x_0)}$ to retain stability even under a contractive recall Jacobian. Appendix \ref{app:anisotropy} conducts a simplified projection-based experiment on the single-layer eigenvector-matched model and finds that internal recall's stable region is substantially more anisotropic than external's. Working directly with models trained on sudoku data, we find in Appendix \ref{app:jacobian-convergence} that internal models with pre-norm never converge to a fixed point, whereas external models can. Together, these theoretical and empirical results suggest that while internal recall may provide expressivity benefits by allowing each layer's attention and MLP to receive different linear combinations of the input and current state, it may also face unfavorable optimization geometry even in the contractive-Jacobian regime.

Appendix~\ref{app:lr-rho} shows that lower learning rates help keep $\rho(\pdv{g(x^*, x_0)}{x^*})$ small throughout training, placing both architectures in this contractive-recall regime. Consistent with the previously-mentioned asymmetry, external recall beats internal recall on all low/medium LR comparisons on hard data, and in 10 out of 12 on validation data (each by at least 2 SEs); the remaining two are statistical ties. 

We note that unnormalized internal recall is the \textit{only} configuration studied here with an identity residual path. Its poor performance suggests that identity residuals, while beneficial in fixed-depth transformers, can be \textbf{actively harmful for recall architectures}. This result is not surprising upon further reflection: identity residuals are beneficial in fixed-depth transformers because they help the gradient propagate through many different layers back to the original input. However, looped models with recall use the same layer at each step and have numerous gradient paths to the input via the recall connection, reducing reliance on an uninterrupted identity path. Meanwhile, as shown above, the identity path with recall can complicate local attraction, making outer normalization beneficial for reaching a fixed point.

By contrast, when $\pdv{g(x^*, x_0)}{x^*}$ is \textit{expansive}, the balance becomes more mixed. Figure \ref{fig:stability_regions} shows how, for $J_g$ large, external recall must actively \textit{shift} the eigenvalues of $J_h$ towards -1, whereas internal recall must merely keep its eigenvalues small and negative. The empirical results here are substantially more task-dependent than the contractive recall scenario; external fails to find any fixed points whatsoever for chess with pre-norm (seeing 0\% accuracy), but is able to find more usable fixed points than internal recall in sudoku with peri-norm. Overall, as $\rho(W_x)$ grows --which Appendix~\ref{app:lr-rho} shows is driven by learning rate -- both architectures move into the expansive-Jacobian regime predicted by Proposition \ref{prop:ext-int} to be unstable, and both degrade accordingly. In several configurations, this expansiveness eliminates usable fixed points entirely.

Taken together, while recall \textit{improves} input-dependence and thus generalization capacity -- the best recall model beats the best autonomous model across all norm and task combinations on hard data -- it is not \textit{sufficient} to ensure a favorable geometry, especially in higher-lr regimes where $\rho(W_x) > 1$\footnote{While we only discuss large $\rho(W_x)$ in the context of high learning rates, it can also occur under poor initializations of the $W_x$ matrix.}. In the next section, we consider outer normalization, which broadens the geometrically-stable regime and, in doing so, enables stronger generalization behavior for both internal and external recall models.

\subsection{Recall with Outer Normalization}\label{sec:recall-outer}
Section \ref{sec:recall-no-outer} showed that without outer normalization, the stable regime is narrow for both recall architectures and especially restrictive for internal recall. Outer normalization broadens this regime: since the outer normalization Jacobians appear as factors throughout the fixed-point Jacobian (see Appendix~\ref{app:explicit} for the explicit form), a contractive $\phi$ can shrink the spectral radius and keep the model within this stable regime -- a mechanism that unnormalized models lack a direct lever for.

In addition, we also require the fixed point to depend non-trivially on $x_0$ -- that is, not merely as basin selection among a discrete set (a key problem of autonomous networks as proven in Proposition \ref{prop:dim0}) but with a well-defined, nonzero local sensitivity $\pdv{x^*}{x_0}$. The following proposition characterizes the benefits of this locally-attractive regime for recall models, regardless of the use of outer normalization. We then examine empirically whether outer normalization helps models enter this regime.

\begin{proposition}\label{prop:recall-outer}
    Suppose we have a generic recall model $x_{t+1} = f(x_t, x_0)$ with fixed starting iterate $e$. Say that $x_T$ converges to $x^*$, a fixed point of the model, with $\rho\left(\pdv{f(x^*, x_0)}{x^*}\right) < 1$. Then, the input-gradient of the iterates converges to the input-gradient of the limiting fixed-point. Formally:\[
    \underset{T \rightarrow \infty}{\lim} \dv{x_T}{x_0} = \dv{}{x_0} \left(\underset{T \rightarrow \infty}{\lim} x_T\right) = \dv{x^*}{x_0} = \left(I - \pdv{f(x^*, x_0)}{x^*}\right)^{-1}\pdv{f(x^*, x_0)}{x_0}
    \]
\end{proposition}
\begin{proof}
    This proof consists of two steps. First, we show that $\pdv{x_*}{x_0} = \left(I - \pdv{f(x^*, x_0)}{x^*}\right)^{-1} \pdv{f(x^*, x_0)}{x_0}$ by showing invertibility of $I - \pdv{f(x^*, x_0)}{x^*}$ and applying the implicit function theorem. Then, we show $\underset{T \rightarrow \infty}{\lim} \dv{x_T}{x_0} = (I - \pdv{f(x^*, x_0)}{x^*})^{-1}\pdv{f(x^*, x_0)}{x_0}$ via a norm-bounding argument that shows the limit of the unrolled left hand side converges to the right side. 
    
    See Appendix \ref{app:input-gradient} for proof details, and Appendix \ref{app:explicit} for explicit Jacobian forms for external and internal models with outer normalization.
\end{proof}

Together, Propositions \ref{prop:ext-int} and \ref{prop:recall-outer} identify a regime that does not exist for autonomous networks. When $\rho\left(\pdv{f(x^*, x_0)}{x^*}\right) < 1$, the limit of gradients $\dv{x_T}{x_0}$ converges to the gradient at the fixed point, a finite value which need not be zero. Thus, fixed points are locally attractive and \textit{vary smoothly} with the input $x_0$, unlike autonomous models which merely choose which of a predetermined set of fixed points to output. In addition, a further consequence of this regime is that the fixed point is \textit{locally independent} of $e$, the initial recall iterate (Definition \ref{def:recall}), as proven in Appendix \ref{app:independent-e}. This initialization-independence is both intuitively helpful ($e$ is problem-independent, so it should not affect the final result) and empirically beneficial: \citet{anil_path_2022} find that path-independence improves reasoning performance on a wide range of problems. 

In Appendix \ref{app:jacobian-convergence}, we find that outer normalization does indeed help achieve this regime. In our tests with external and internal models on sudoku problems, all converged hidden states had full-loop Jacobian spectral radii below 1. In addition, removing the outer-normalization Jacobian factors at the same states substantially increased the spectral radius, justifying our argument that outer normalization provides benefits by shrinking the full-loop spectral radius rather than simply ensuring bounded hidden states. These qualitative results also held for two-block models, providing evidence that outer normalization provides a similar role in more complicated architectures. In Table \ref{tab:hard}, we find that these benefits indeed transfer to real-world tasks, with the best normalized looped model outperforming the best unnormalized looped model by 13 and 4.5 points on sudoku and chess, respectively\footnote{Both unnormalized and normalized models saturate prefix sums, though autonomous models never reach non-zero hard accuracy.}. 

We note that although normalized models generally outperform un-normalized ones, the exact best architectural choices -- post vs gru norm, internal vs external recall -- differ across tasks, since other factors can dominate once the axes are satisfied (e.g. exponential token clustering in post-norm \citep{karagodin2025normalizationattentiondynamics, geiping_scaling_2025}, or representation capacity differences between internal and external recall). Thus, the axes should be treated as \textit{necessary conditions rather than sufficient}: models without them degrade on validation and/or hard data, but models with them do not always succeed across tasks. We leave possible sufficient conditions for future work.

\section{Conclusion}
This paper introduced a framework for understanding when looped transformers learn generalizable, scalable algorithms rather than memorize training-depth solutions, organized around three axes of stability: local attraction, input-dependence, and geometry. We analyzed three architectural choices -- the existence of recall, the placement of recall, and the presence of outer normalization -- and found that each affects a distinct subset of the axes in ways that predict downstream performance. 

Autonomous networks exhibit extremely weak input-dependence (Propositions \ref{prop:fixed-jacobian} - \ref{prop:instability-autonomous}); recall without outer normalization possesses input-dependent fixed points but suffers from a narrow, anisotropic stability geometry (Figure \ref{fig:stability_regions}, Proposition \ref{prop:ext-int}); outer normalization broadens the stable regime and improves stability across all three axes. Our empirical results across chess, sudoku, and prefix-sums track these predictions. Internal recall, a novel placement variant, performs poorly with an identity residual but is competitive with external recall once outer normalization is added -- though the relative ranking across tasks remains unresolved, and whether these findings extend to larger models is an object for future work. Our stability conditions are also local; while they provide for attraction near a fixed point, they do not characterize its basin size, so they do not ensure that $x_0$ reaches it. While we primarily apply these axes to describe and justify existing architectural choices, they also suggest a strategy for developing future architectures designed for local attraction, input-dependence, and favorable geometry from the start. 

\begin{ack}
We thank our advisors at Brown University: Peyam Tabrizian (Applied Math Department), and Andras Zsom (Department of Data Science).

This work was generously supported via a grant from the Brown University Library. Additional compute for the rebuttal and camera-ready phases was provided by Liquid AI.
\end{ack}

\newpage
\bibliographystyle{plainnat}
\bibliography{references}

\newpage
\appendix

\section{Proofs}
\paragraph{Notation.} Unless stated otherwise, throughout this appendix we treat the parameters $\theta$ fixed as a constant of the problem. We write $\dv{}{x_0}$ for total derivatives that consider all dependencies of $x_0$ through the unrolled recurrence. For derivatives of single-argument functions, we similarly use total derivatives: for example, $\dv{h(g(x_t, x_0))}{g(x_t, x_0)}$ denotes the derivative of $h$ with respect to its argument. We reserve $\partial$ for multi-argument functions: for example, $\pdv{g(x_t, x_0)}{x_0}$ denotes the partial derivative of $g$ with respect to the second slot, holding $x_t$ constant. 

This distinction matters because several proofs below involve chain rule computations where the same variable appears as both a direct argument of a function and an indirect dependency through earlier states. We specify notation here to provide an unambiguous meaning throughout algebra-heavy proofs.

\subsection{Weakness of Transversality Assumption}\label{app:transversality}

As a concrete illustration, consider a linear network $f: \mathbb{R}^d \rightarrow \mathbb{R}^d$ defined as $f(x) = Ax + b$. Then $\pdv{f}{b} = I$, which is full-rank regardless of $A$ or $x$, so transversality applies trivially. Below, we prove that a similar argument on A extends this result to length-$L$ inputs whenever the tokens are linearly independent and $L \leq d$. 

\begin{proposition}\label{prop:transversality-weak}
    Consider a network $f: \mathbb{R}^{d \times L} \times \mathbb{R}^{d \times d} \times \mathbb{R}^d \rightarrow \mathbb{R}^{d \times L}$, linear in its first argument, defined as $f(X, A, b) = AX + b\mathbf{1}^T$. Assume that $L \leq d$ and that the columns of $X$ are linearly independent (rank $L$). Then, $\pdv{f(X, A, b)}{A}$, a $dL \times d^2$ matrix, is full-rank.
\end{proposition}
\begin{proof}    
    Fix $X$ and $b$. We note that $f(X, A + T, b) = AX + TX + b\mathbf{1}^T$. Thus, \begin{align*}
        \underset{T \rightarrow 0}{\lim} \frac{f(X, A + T, b) - f(X, A, b) - TX}{\norm{T}} &= \\
        \underset{T \rightarrow 0}{\lim} \frac{AX + TX + b\mathbf{1}^T - AX - b\mathbf{1}^T - TX}{\norm{T}} &= \\
        0
    \end{align*}
    And thus $Df(T) = TX$ is the partial derivative of $f$ with respect to A (though, it does not \textit{depend} on A). Rather than directly analyze the $dL \times d^2$ matrix, we instead show this map $Df: \mathbb{R}^{d \times d} \rightarrow \mathbb{R}^{d\times L}$ is surjective. Consider arbitrary $Y \in \mathbb{R}^{d \times L}$. Since $X$ has full column rank, $X^TX$ is invertible. Then, let $T$ = $Y(X^TX)^{-1}X^T \in \mathbb{R}^{d \times d}$ so that $Df(T) = Y(X^TX)^{-1}X^TX = Y$. Since Y was arbitrary, we have that the partial derivative is surjective, and thus $\pdv{f(X, A, b)}{A}$ is full-rank regardless of the starting matrix $A$. Thus, the transversality assumption holds at fixed points $X$ with linearly independent columns and $L \leq d$.
\end{proof}

We note two facts about this proof. First, it does not apply to our experiments on prefix-sums, given that they use 512-bit sequences with a hidden dimension of just 256. As the proof is merely a sufficient condition, not a necessary one -- the combined state-parameter Jacobian could be full-rank, even if $\pdv{f(X, A, b)}{A}$ alone is not -- this proof simply gives no information in that case. Second, this example does not establish transversality for transformer architectures: while greater model flexibility and parameter count may help in ensuring the required independent directions, this proof remains a heuristic rather than a guarantee for nonlinear models.

\subsection{Joint Fixed Points in Autonomous Networks}\label{app:joint-fixed}
Proposition \ref{prop:dim0} proves that for almost all parameterizations $\theta$, the set of fixed points of $f_\theta$ is a dimension-zero manifold (finite if $x_t$ is compact). However, many looped models include a final \enquote{head} layer after concluding the looped layers, which may map multiple distinct values to the same value. For example, applying a normalization after the looped layers maps all $\lambda x_t$ to the same point. Thus, it is possible for there to exist points $x_t$ wherein $f_\theta(x_t) \neq x_t$, but $w_\theta(f_\theta(x_t)) = w_\theta(x_t)$ for some head $w_\theta: \mathbb{R}^{d \times L} \rightarrow \mathbb{R}^{d \times L}$. We call these points \enquote{joint fixed points}. As before we define $h_\theta = w_\theta(f_\theta(x_t)) - w_\theta(x_t)$ such that $h_\theta(x_t) = 0 \leftrightarrow x_t$ is a joint fixed point. If we augment Assumption \ref{assum:transversal} so that h is transversal to 0 instead, we get that $h_\theta^{-1}(0)$ (the set of all joint fixed points) is a dimension-zero set (finite if $x_t$ is bounded). This proves that, while head layers may be useful in some formulations of looped models, for autonomous models, they do not change the dimension of the set of fixed points prior to applying the loss.

\subsection{Gradient Instability for Autonomous Looped Networks}\label{app:grad-instability}

\begin{proof}
    Write $\rho^* = \rho(J_{f_\theta}(x^*))$. Fix $\alpha \in (\rho^*, 1)$, and let $\epsilon = \frac{\alpha - \rho^*}{2}$. By \cite{Horn_Johnson_1985} Lemma 5.6.10, there exists a submultiplicative matrix norm, written $\norm{.}_*$, such that $\norm{J_{f_\theta}(x^*)}_* \leq \frac{\alpha + \rho^*}{2} = \rho^* + \epsilon$. Assuming smoothness of $f$, $x \rightarrow J_{f_\theta}(x)$ is continuous. Since all norms are equivalent in finite-dimensional vector spaces, there exists a $\delta$ such that for $\norm{x - x^*} < \delta$, $\norm{J_{f_\theta}(x) - J_{f_\theta}(x^*)}_* < \epsilon$. Since $x_t \rightarrow x^*$, we also know that $\exists N \; s.t. \; \forall t > N, \norm{x_t - x^*} < \delta$. Combined, we have \[
    \exists N \; s.t. \; \forall t > N, \norm{J_{f_\theta}(x_t) - J_{f_\theta}(x^*)}_*  < \epsilon
    \]
    Using the triangle inequality we get 
    \begin{align*}
        \norm{J_{f_\theta}(x_t)}_* &\leq \\
        \norm{J_{f_\theta}(x_t) - J_{f_\theta}(x^*)}_* + \norm{J_{f_\theta}(x^*)}_* &< \\
        \epsilon + \rho^* &= \\
        \alpha
    \end{align*}

    We can now plug this into the gradient as T grows. Since the network is autonomous, each iterate satisfies $x_{t+1} = f_\theta(x_t)$. By the chain rule, \begin{align*}
    \dv{x_T}{x_0} &= \\
    \prod_{t=0}^{T-1} J_{f_\theta}(x_t) &= \\ \prod_{t=N+1}^{T-1} J_{f_\theta}(x_t) \; \prod_{t=0}^{N} J_{f_\theta}(x_t) &< \\
    \alpha^{T - N - 2} {\prod_{t=0}^{N} J_{f_\theta}(x_t)}
    \end{align*}

    Since $\alpha < 1$, as $T \rightarrow \infty$, this vanishes exponentially with T.
\end{proof}

\subsection{Stability Cases for Autonomous Looped 
Networks}\label{app:cases}

Proposition \ref{prop:instability-autonomous} finds that the behavior of $\dv{x_T}{x_0}$ depends strongly on the spectral radius of $J_{f_\theta} = \pdv{f_\theta(x^*)}{x^*}$. Here, we split the spectral radius into three cases and analyze the regime of each.

\begin{enumerate}
    \item \textbf{Case 1: $\bm{\rho(J_{f_\theta}(x^*)) < 1}$}. Gelfand's formula implies that $\norm{J_{f_\theta}(x^*)^{T-K}}$ decays exponentially with $T$. In particular, $\forall \alpha \in (\rho(J_{f_\theta}(x^*)), 1), \exists C_\alpha \in \mathbb{R}^+$ with $\norm{J_{f_\theta}(x^*)^{T-K}} \leq C_\alpha \alpha^{T-K}$. Thus, we have that \[\norm{\dv{x_T}{x_0}} \leq \norm{J_{f_\theta}(x^*)}^{T-K}\norm{A_{(x_0, \theta)}} \leq C_{\alpha, x_0, \theta} \alpha^{T-K}\] which decays exponentially with $T$. This is \textit{input-independence}: as $T \to \infty$, $\left\|\dv{x_T}{x_0}\right\| \to 0$ exponentially, meaning $x_T$ becomes asymptotically insensitive to the initial state $x_0$. Since $x_0$ is the only information provided to the model $f_\theta$, this means the model essentially \enquote{forgets} the information it was given.
    \item \textbf{Case 2: $\bm{\rho(J_{f_\theta}(x^*)) > 1}$}. This implies that there is at least one eigenvalue with modulus greater than 1. In this scenario, the center-stable manifold theorem (\citep[p.65]{Shub2010-xz}) states the local center-stable manifold around $x^*$ has dimension $< dL$, and is thus measure zero. Thus among points in a sufficiently small neighborhood of $x^*$, almost every point eventually leaves the neighborhood, with at least one direction repelling away from $x^*$. Figure \ref{fig:phase-portraits} provides two example systems, one with all $|\lambda| < 1$ and one with an eigenvalue greater than 1; in the latter case, any starting point except the $y = 0$ line is repelled away from the fixed point. Unlike the $\rho(J_{f_\theta}) < 1$ scenario, wherein fixed points are easily reached but fail to depend on the input, fixed points in this case are practically never reached at all.
    \item \textbf{Case 3: $\bm{\rho(J_{f_\theta}(x^*)) = 1}$}. This case provides very little useful information, given that $x^*$ is a non-hyperbolic fixed point, of which much less can be derived. 
    
    We conduct two short analyses on this regime. First, we consider what occurs as the Jacobian approaches 1 from a locally-attractive regime; i.e. as $\rho(J_{f_\theta}(x^*)) \rightarrow 1^-$. Under this scenario, we use the fact that $\rho(J_{f_\theta}(x^*)) < 1$ along with equation \eqref{eq:jacobianf} from Proposition \ref{prop:fixed-jacobian} to note that $\pdv{g(x^*, \theta)}{x^*}$ is nonsingular for $g(x^*, \theta) = f_\theta(x^*) - x^*$. Thus the implicit function theorem applies and we can write $x^*$ as a smooth function of $\theta$ nearby the zero point. \\
    Note that our gradient notation shifts at this point: here we treat $\theta$ as the key variable and let $\dv{g(x^*(\theta), \theta)}{\theta}$ represent the total derivative with respect to $\theta$. We can then write that the total gradient $\dv{g(x^*(\theta), \theta)}{\theta}$ = 0 at the zero point (due to the existence of a neighborhood with $g(x^*(\theta), \theta) = 0$), so we use the chain rule to get \begin{align*}
        \pdv{g(x^*, \theta)}{x^*} \dv{x^*(\theta)}{\theta} + \pdv{g(x^*, \theta)}{\theta} = 0 &\rightarrow \\
        \dv{x^*(\theta)}{\theta} = -\left(\pdv{g(x^*, \theta)}{x^*}\right)^{-1}\pdv{g(x^*, \theta)}{\theta} &\rightarrow \\
        \dv{x^*(\theta)}{\theta} = -\left(\pdv{f_\theta(x^*)}{x^*} - I\right)^{-1}\pdv{f_\theta(x^*)}{\theta} 
    \end{align*}

    Thus, if the largest eigenvalue of the Jacobian is near-1 (which occurs when the sublayer is near-zero), the smallest eigenvalue of $\pdv{f_\theta(x^*)}{x^*} - I \rightarrow 0$ and thus $\dv{x^*(\theta)}{\theta}$ explodes in magnitude (as long as it is not zeroed by the right-multiplication of $\pdv{f_\theta(x^*)}{\theta}$). Therefore the \textit{parameter} gradient becomes arbitrarily unstable in this incentivized regime. 

    We simultaneously consider the \textit{speed} of convergence in the $\rho(J_{f_\theta}(x^*)) = 1$ regime. \citet{Shub2010-xz} shows that in such a regime, fixed-point convergence occurs no faster than \textit{sub-exponentially}, and thus may require far more compute than the exponential convergence of fixed points with $\rho(J_{f_\theta}(x^*)) < 1$. Thus, in this regime, the model has a greater probability to fail to reach a fixed point during training in the first place, making generalization more difficult. 
\end{enumerate}

\subsection{Stability Regimes of Recall Models without Outer Normalization}\label{app:stability-no-outer}

\begin{proof}
    For external recall, we have that \begin{align*}
        \dv{x_{t+1}}{x_t} &= \\
        \dv{x_{t+1}}{z_t} \dv{z_t}{x_t} &= \\
        \left(I + \dv{h_2(z_t)}{z_t}\right) \left(\pdv{g(x_t, x_0)}{x_t} + \dv{h_1(g(x_t, x_0))}{g(x_t, x_0)}\pdv{g(x_t, x_0)}{x_t} \right) &= \\
        \left(I + \dv{h_2(z_t)}{z_t}\right) \left(I + \dv{h_1(g(x_t, x_0))}{g(x_t, x_0)} \right)\pdv{g(x_t, x_0)}{x_t}
    \end{align*}
    Once a fixed-point is reached, this becomes
    \[
    M_{\text{ext}} = \left(I + \dv{h_2(z^*)}{z^*}\right) \left(I + \dv{h_1(g(x^*, x_0))}{g(x^*, x_0)} \right)\pdv{g(x^*, x_0)}{x^*}
    \]
    As explained in Section \ref{sec:autonomous} and Appendix \ref{app:cases}, a fixed point can only be practically reachable (given floating-point perturbations) if it is locally attractive, which requires $\rho\left(\pdv{f_\theta(x^*, x_0)}{x^*}\right) = \rho(M_{\text{ext}}) \leq 1$. 

    For internal recall, we have that \begin{align*}
        \dv{x_{t+1}}{x_t} &= \\
        \dv{x_{t+1}}{z_t} \dv{z_t}{x_t} &= \\
        \left(I + \dv{h_2(g(z_t, x_0))}{g(z_t, x_0)} \pdv{g(z_t, x_0)}{z_t} \right) \left(I + \dv{h_1(g(x_t, x_0))}{g(x_t, x_0)} \pdv{g(x_t, x_0)}{x_t} \right)
    \end{align*}
    Once a fixed point is reached, this becomes \[
    M_{\text{int}} = \left(I + \dv{h_2(g(z^*, x_0))}{g(z^*, x_0)} \pdv{g(z^*, x_0)}{z^*} \right) \left(I + \dv{h_1(g(x^*, x_0))}{g(x^*, x_0)} \pdv{g(x^*, x_0)}{x^*} \right)
    \] 
    
    The equivalent argument as in external recall requires $\rho(M_{\text{int}}) \leq 1$ for internal recall models to be locally stable.
\end{proof}

\subsection{Stability of Recall Models with Outer Normalization}\label{app:input-gradient}
We prove this in two steps.\\
    
\textbf{Step 1:} $\dv{x^*}{x_0} = \left(I - \pdv{f(x^*, x_0)}{x^*}\right)^{-1}\pdv{f(x^*, x_0)}{x_0}$\\

To begin with this proof, we must first show that the left term in this equation is well-defined and differentiable in the first place. We define the function $F(x_t, x_0) = f(x_t, x_0) - x_t$ so that fixed points of $f$ are zeros of $F$. At a fixed point $x^*$, we have (by assumption) that $\rho\left(\pdv{f(x^*, x_0)}{x^*}\right) < 1$. Thus, 1 is not an eigenvalue of $\pdv{f(x^*, x_0)}{x^*}$. In that case, 0 is not an eigenvalue of \[
\pdv{F(x^*, x_0)}{x^*} = \pdv{f(x^*, x_0)}{x^*} - I\] and thus $J_F = \pdv{F(x^*, x_0)}{x^*}$ is invertible, so the implicit function theorem applies. Thus, there exists a differentiable function $x^*(x_0)$ in a neighborhood of $x_0$ such that $F(x^*(x_0), x_0) = 0$ for all points on that neighborhood. Further, we have that \[
\dv{x^*}{x_0} = -\left(\pdv{F(x^*(x_0), x_0)}{x^*}\right)^{-1}\pdv{F(x^*(x_0), x_0)}{x_0}
\] and \[
\dv{x^*}{x_0} = \left(I - \pdv{f(x^*(x_0), x_0)}{x^*}\right)^{-1}\pdv{f(x^*(x_0), x_0)}{x_0}
\]

\textbf{Step 2:} $\underset{T \rightarrow \infty}{\lim} \dv{x_T}{x_0} = \left(I - \pdv{f(x^*(x_0), x_0)}{x^*}\right)^{-1}\pdv{f(x^*(x_0), x_0)}{x_0}$ \\

From our assumption, we have that $\rho\left(\pdv{f(x^*, x_0)}{x^*}\right) < 1$. While useful for the prior step, here we instead need to bound the \textit{norm} of the Jacobian. \citet{Horn_Johnson_1985} Lemma 5.6.10 tells us that for any $\epsilon > 0$, matrix A, there exists a norm $*$ such that $\norm{A}_* \leq \rho(A) + \epsilon$. First defining $\rho = \rho\left(\pdv{f(x^*, x_0)}{x^*}\right)$ for notational simplicity, we choose $\epsilon = \frac{1 - \rho}{4}$, so that $\norm{\pdv{f(x^*, x_0)}{x^*}}_* < \frac{1 + 3\rho}{4}$.

We continue with two stated facts: first, that $x_T \rightarrow x^*$, and second, that $f(x_t, x_0)$ is smooth in both of its arguments. Thus, $\pdv{f(x_t, x_0)}{x_t}$ is continuous in $x_t$, and since norms are a continuous function of a matrix, $\norm{\pdv{f(x_t, x_0)}{x_t}}_*$ is also continuous in $x_t$. Applying this to the fixed point $x^*$, we have that for any $\epsilon > 0, \exists \delta > 0 \; s.t. \; \norm{x^\prime - x^*} < \delta \implies \bigg| \norm{\pdv{f(x^\prime, x_0)}{x^\prime}}_* - \norm{\pdv{f(x^*, x_0)}{x^*}}_*\bigg| < \epsilon$. We choose $\epsilon = \frac{1 - \rho}{4}$ so that $\norm{\pdv{f(x^\prime, x_0)}{x^\prime}}_* < \frac{1 + 3\rho}{4} + \frac{1 - \rho}{4} = \frac{1 + \rho}{2}$.

Similarly, since $x_T$ converges to $x^*$, we know that for $\epsilon > 0, \exists N s.t. \forall n > N, \norm{x_n - x^*} < \epsilon$. We choose $\epsilon$ as $\delta$ found earlier, and get that $\exists N \; s.t. \; \forall n > N, \norm{\pdv{f(x_n, x_0)}{x_n}}_* < \frac{1 + \rho}{2}$. We can follow the same argument for $\norm{\pdv{f(x_n, x_0)}{x_0} - \pdv{f(x^*, x_0)}{x_0}}_* < \epsilon$, and choose N large enough so that for small, prechosen $\alpha$, we have all three conditions fulfilled: \begin{enumerate}
    \item $\norm{\pdv{f(x_n, x_0)}{x_n}}_* < \frac{1 + \rho}{2}$
    \item $\norm{\pdv{f(x_n, x_0)}{x_0} - \pdv{f(x^*, x_0)}{x_0}}_* < \alpha$
    \item $\norm{\pdv{f(x_n, x_0)}{x_n} - \pdv{f(x^*, x_0)}{x^*}}_* < \alpha$
\end{enumerate}

With this in mind, we consider the recurrence $\dv{x_t}{x_0} = \pdv{x_t}{x_{t-1}} \dv{x_{t-1}}{x_0} + \pdv{x_t}{x_0}$, with $x_t$ both a direct function of $x_0$ (recall) and an indirect one through $x_{t-1}$. Unrolling this to zero, we get: \[
\dv{x_t}{x_0} = \sum_{i=1}^t \left(\prod_{j=i+1}^t \pdv{x_j}{x_{j-1}}\right) \pdv{x_i}{x_0}
\]
Since we care about the result as $t \rightarrow \infty$, we consider what occurs as $t > N$. We split up the sum into before and after N: \begin{align*}
    \sum_{i=1}^T \left( \prod_{j = i + 1}^T \pdv{x_j}{x_{j-1}}  \right) \pdv{x_i}{x_0} &= \\
    \sum_{i=1}^N \left(\prod_{j=i}^N \pdv{x_j}{x_{j-1}} \right) 
    \left(\prod_{j=N+1}^T \pdv{x_j}{x_{j-1}} \right) \pdv{x_i}{x_0} + \sum_{i=N+1}^T \left(\prod_{j=i+1}^T \pdv{x_j}{x_{j-1}}\right) \pdv{x_i}{x_0}
\end{align*}
And since $\norm{\pdv{x_i}{x_{i-1}}}_* = \norm{\pdv{f(x_i, x_0)}{x_i}}_* < \frac{1 + \rho}{2}$ as long as $i > N$, the first term goes to 0 as $T \rightarrow \infty$. We are left with only the final term. 

We now wish to prove that \[
\underset{T \rightarrow \infty}{\lim} \sum_{i=N+1}^T \left(\prod_{j=i+1}^T \pdv{x_j}{x_{j-1}}\right) \pdv{x_i}{x_0} = \left(I - \pdv{f(x^*, x_0)}{x^*}\right)^{-1} \pdv{f(x^*, x_0)}{x_0}
\]

We proceed by first naming matrices for ease of reading. $A_t = \pdv{x_t}{x_{t-1}}, B_t = \pdv{x_t}{x_0}$, and $A_* = \pdv{f(x^*, x_0)}{x^*}$ with $B_* = \pdv{f(x^*, x_0)}{x_0}$. We now reindex the sum backwards $i \rightarrow T - i$ and substitute these names to get:
\[
S_T = \sum_{i=0}^{T-N-1} \left(\prod_{j=T - i + 1}^T A_j\right) B_{T-i}
\]
We compare this sum with the sum of the fixed-point matrices, namely \[
S^*_T = \sum_{i=0}^{T-N-1} A_*^i B_*
\]
Note that as $T \rightarrow \infty$, as long as $\norm{A_*} < 1$ for some norm (which we previously proved for norm $*$), $S^*_T$ converges to $\left(I - \pdv{f(x^*, x_0)}{x^*}\right)^{-1} \pdv{f(x^*, x_0)}{x_0}$.

We wish to prove that $\lim_{T \rightarrow \infty} \norm{S_T^* - S_T}_* = 0$; i.e. the Jacobians are the same in the limit. By the triangle inequality, we have that \[
\norm{S^*_T - S_T}_* \leq \sum_{i=0}^{T-N-1} \norm{\left(\prod_{j=T - i + 1}^T A_j\right) B_{T-i} - A_*^i B_*}_*
\]
We can add and subtract $A_*^iB_{T-i}$ and simplify to get:
\begin{align*}
     \sum_{i=0}^{T-N-1} \norm{\left(\prod_{j=T - i + 1}^T A_j\right) B_{T-i} - A_*^i B_*}_* &= \\
      \sum_{i=0}^{T-N-1} \norm{\left(\prod_{j=T - i + 1}^T A_j\right) B_{T-i} - A_*^iB_{T-i} + A_*^iB_{T-i} - A_*^i B_*}_* &= \\
      \sum_{i=0}^{T-N-1} \norm{\left(\left(\prod_{j=T - i + 1}^T A_j\right) - A_*^i\right) B_{T-i} + A_*^i(B_{T-i} - B_*)}_* &\leq \\ 
      \sum_{i=0}^{T-N-1} \left(\norm{\left(\left(\prod_{j=T - i + 1}^T A_j\right) - A_*^i\right)}_*\norm{B_{T-i}}_* + \norm{A_*}_*^i\norm{B_{T-i} - B_*}_*\right)
\end{align*}
We now bound each of the individual norms. First, note that i goes from 0 to $T - N - 1$, so $T - i > N$ across the entire sum. Thus, $\norm{B_{T-i} - B_*}_* < \alpha$ as described earlier. Also, we have that $\norm{A_*}_* < \frac{1 + \rho}{2} = \frac{1 + \rho}{2}$ (renaming), so that $\norm{A_*}^i_* < (\frac{1 + \rho}{2})^i$. By the triangle inequality, we have that \[
\norm{B_{T-i}}_* = \norm{B_{T - i} - B_* + B_*}_* \leq \norm{B_{T-i} - B_*}_* + \norm{B_*}_* < \alpha + \norm{B_*}_*
\] 
For simplicity, we notate $M = \alpha + \norm{B_*}_*$, a constant regardless of T or i. Thus, the only term left to bound is the very first term. To do so, we prove a short lemma by induction:

\begin{lemma}\label{lem:telescope}
    For matrices $(A_n)_{n=1}^K, B$ that can be multiplied, \[
    \prod_{i=1}^K A_i - B^K = \sum_{i=1}^K \left(\prod_{j=1}^{i-1} A_j\right) (A_i - B)B^{K - i}
    \]
\end{lemma}
\begin{proof}
    We prove by induction. 
    
    With K = 1, we have $A_1 - B$ = $I(A_1 - B)B^0 = A_1 - B$, since the product term is the identity when the bottom is less than the top. 

    Now we suppose this is true for K. Then, for K + 1, we have \begin{align*}
        \prod_{i=1}^{K+1} A_i - B^{K+1} &= \\
        \left(\prod_{i=1}^{K} A_i\right) A_{K+1} - B^KB &= \\
        \left(\prod_{i=1}^{K} A_i\right) A_{K+1} - \left(\prod_{i=1}^{K} A_i\right)B + \left(\prod_{i=1}^{K} A_i\right)B - B^KB &= \\
        \left(\prod_{i=1}^{K} A_i\right)(A_{K+1} - B) + \left(\prod_{i=1}^{K} A_i - B^K\right)B &= \\
        \left(\prod_{i=1}^{K} A_i\right)(A_{K+1} - B) + \sum_{i=1}^K \left(\prod_{j=1}^{i-1} A_j\right) (A_i - B)B^{K - i}B &= \\
        \left(\prod_{i=1}^{K} A_i\right)(A_{K+1} - B) + \sum_{i=1}^K \left(\prod_{j=1}^{i-1} A_j\right) (A_i - B)B^{(K + 1) - i} &= \\
        \sum_{i=1}^{K+1} \left(\prod_{j=1}^{i-1} A_j\right) (A_i - B)B^{(K + 1) - i}
    \end{align*}
\end{proof}

We then apply Lemma \ref{lem:telescope} to the first term and get that \begin{align*}
    \norm{\left(\left(\prod_{j=T - i + 1}^T A_j\right) - A_*^i\right)}_* &= \\
    \norm{\left(\left(\prod_{j=1}^{i} A_{j+T-i}\right) - A_*^i\right)}_* &= \\
    \norm{\sum_{j=1}^i \left(\prod_{k=1}^{j-1} A_{k+T-i}\right) (A_{j+T-i} - A_*)A_*^{i - j}}_* &\leq \\
    \sum_{j=1}^i \norm{\prod_{k=1}^{j-1} A_{k+T-i}(A_{j+T-i} - A_*)A_*^{i - j}}_* &\leq \\
    \sum_{j=1}^i \prod_{k=1}^{j-1} \norm{ A_{k+T-i}}_* \norm{(A_{j+T-i} - A_*)}_*\norm{A_*}_*^{i - j}
\end{align*}

Each factor in the first term has index greater than N, so we can bound $\norm{ A_{k+T-i}}_* < \frac{1 + \rho}{2}$. The latter term is similarly bound: $\norm{A_*}_*^{i - j} < (\frac{1 + 3\rho}{4})^{i-j} < (\frac{1 + \rho}{2})^{i-j}$. The middle term is bound by $\alpha$, as it also has index greater than N. Thus the entire Jacobian is bounded by
\begin{align*}
    \sum_{j=1}^i \left(\prod_{k=1}^{j-1} \frac{1 + \rho}{2} \right) \alpha\left(\frac{1 + \rho}{2}\right)^{i - j} &= \\
    \sum_{j=1}^i \left(\frac{1 + \rho}{2}\right)^{j - 1} \alpha \left(\frac{1 + \rho}{2}\right)^{i - j} &= \\
    \sum_{j=1}^i \left(\frac{1 + \rho}{2}\right)^{i - 1} \alpha &\leq \\
    i \left(\frac{1 + \rho}{2}\right)^i \alpha
\end{align*}

Going back to the original series, we have that \begin{align*}
    \sum_{i=0}^{T-N-1} \left(\norm{\left(\left(\prod_{j=T - i + 1}^T A_j\right) - A_*^i\right)}_*\norm{B_{T-i}}_* + \norm{A_*}_*^i\norm{B_{T-i} - B_*}_*\right) &\leq \\
    \sum_{i=0}^{T-N-1} \left( \left(i \left(\frac{1 + \rho}{2}\right)^{i-1} \alpha M\right) + \alpha \left(\frac{1 + \rho}{2}\right)^i \right) &= \\
    \alpha \sum_{i=0}^{T-N-1} \left( \left(i \left(\frac{1 + \rho}{2}\right)^{i-1} M\right) + \left(\frac{1 + \rho}{2}\right)^i \right)
\end{align*}
Both terms converge to a positive number, since the right term is a geometric series with $|r| < 1$ and the first is the derivative of a geometric series with converged sum equal to $\frac{1}{(1 - r)^2}$ for $|r| < 1$. Thus, the sum is bounded by a constant, and the series is bounded by $C\alpha$. 

We then have by the triangle inequality that \[
\underset{T \rightarrow \infty}{\lim} \norm{S_T - (I - A_*)^{-1}B_*}_* \leq \underset{T \rightarrow \infty}{\lim} \norm{S_T - S_T^*}_* + \underset{T \rightarrow \infty}{\lim} \norm{S_T^* - (I - A_*)^{-1}B_*}_*\]

The first term can be made arbitrarily small by setting $\alpha \rightarrow 0$ (and thus choosing large T and N), and the latter term is equal to 0. Thus, 
\[
\lim_{T \rightarrow \infty} \dv{x_T}{x_0} = \left(I - \pdv{f(x^*, x_0)}{x^*}\right)^{-1} \pdv{f(x^*, x_0)}{x_0}
\]
completing the proof.

\subsection{Explicit Jacobian Forms for External/Internal Recall}\label{app:explicit}
Here, we provide explicit forms of the Jacobians $\pdv{x_t}{x_{t-1}} = \pdv{f(x_{t-1}, x_0)}{x_{t-1}} $ and $\pdv{x_t}{x_0} = \pdv{f(x_{t-1}, x_0)}{x_0}$ for both internal and external recall. For notational simplicity, we let $\Tilde{z}_t, \Tilde{x}_t$ represent the result prior to applying $\phi_1, \phi_2$ respectively; i.e. $z_t = \phi_1(\Tilde{z}_t)$ and $x_t = \phi_2(\Tilde{x_t})$. Note below the presence of the outer normalization Jacobians in the explicit forms, which can shrink the spectral radius if they are small. 

For ease of memory, we redefine external and internal recall below:

External:
\begin{align*}
    x_{t+1} &= \phi_2(z_t + h_2(z_t)) \\
    z_t &= \phi_1(g(x_t, x_0) + h_1(g(x_t, x_0)))
\end{align*}

Internal:
\begin{align*}
    x_{t+1} = \phi_2(z_t + h_2(g(z_t, x_0))) \\
    z_t = \phi_1(x_t + h_1(g(x_t, x_0)))
\end{align*} 

\textbf{External Recall:} 
    \begin{align*}
        \dv{x_t}{x_{t-1}} &= \\
        \dv{x_t}{\Tilde{x}_t}\dv{\Tilde{x}_t}{z_{t-1}}\dv{z_{t-1}}{\Tilde{z}_{t-1}}\pdv{\Tilde{z}_{t-1}}{x_{t-1}} &= \\
        \dv{\phi_2(\Tilde{x}_t)}{\Tilde{x}_t} \left(I + \dv{h_2(z_{t-1})}{z_{t-1}} \right) \dv{\phi_1(\Tilde{z}_{t-1})}{\Tilde{z}_{t-1}} \left(\pdv{g(x_{t-1}, x_0)}{x_{t-1}} + \dv{h_1(g(x_{t-1}, x_0))}{g(x_{t-1}, x_0)}\pdv{g(x_{t-1}, x_0)}{x_{t-1}} \right) &= \\
        \dv{\phi_2(\Tilde{x}_t)}{\Tilde{x}_t} \left(I + \dv{h_2(z_{t-1})}{z_{t-1}} \right) \dv{\phi_1(\Tilde{z}_{t-1})}{\Tilde{z}_{t-1}} \left(I + \dv{h_1(g(x_{t-1}, x_0))}{g(x_{t-1}, x_0)} \right) \pdv{g(x_{t-1}, x_0)}{x_{t-1}}
    \end{align*}
    \textbf{Internal Recall:}
    \begin{align*}
        \dv{x_t}{x_{t-1}} &= \\
        \dv{x_t}{\Tilde{x}_t}\pdv{\Tilde{x}_t}{z_{t-1}}\dv{z_{t-1}}{\Tilde{z}_{t-1}}\pdv{\Tilde{z}_{t-1}}{x_{t-1}} &= \\
        \dv{\phi_2(\Tilde{x}_t)}{\Tilde{x}_t} \left(I + \dv{h_2(g(z_{t-1}, x_0))}{g(z_{t-1}, x_0)} \pdv{g(z_{t-1}, x_0)}{z_{t-1}}\right)\dv{\phi_1(\Tilde{z}_{t-1})}{\Tilde{z}_{t-1}} \cdot \\
        \left(I + \dv{h_1(g(x_{t-1}, x_0))}{g(x_{t-1}, x_0)} \pdv{g(x_{t-1}, x_0)}{x_{t-1}}\right)
    \end{align*}

    \subsubsection{Jacobian of $\pdv{x_t}{x_0}$}

    \paragraph{External Recall:}
    \begin{align*}
        \pdv{x_t}{x_0} &= \\
        \dv{x_t}{\Tilde{x}_t} \dv{\Tilde{x}_t}{z_{t-1}} \dv{z_{t-1}}{\Tilde{z}_{t-1}}\pdv{\Tilde{z}_{t-1}}{x_0} &= \\
        \dv{\phi_2(\Tilde{x}_t)}{\Tilde{x}_t} 
        \left(I + \dv{h_2(z_{t-1})}{z_{t-1}}\right)\dv{\phi_1(\Tilde{z}_{t-1})}{\Tilde{z}_{t-1}}\left(\pdv{g(x_{t-1}, x_0)}{x_0} + \dv{h_1(g(x_{t-1}, x_0)}{g(x_{t-1}, x_0)}\pdv{g(x_{t-1}, x_0)}{x_0} \right) &= \\
        \dv{\phi_2(\Tilde{x}_t)}{\Tilde{x}_t} 
        \left(I + \dv{h_2(z_{t-1})}{z_{t-1}}\right)\dv{\phi_1(\Tilde{z}_{t-1})}{\Tilde{z}_{t-1}}\left(I + \dv{h_1(g(x_{t-1}, x_0)}{g(x_{t-1}, x_0)}\right) \pdv{g(x_{t-1}, x_0)}{x_0}
    \end{align*}
    
    \paragraph{Internal Recall:}
    \begin{align*}
        \pdv{x_t}{x_0} &= \\
        \dv{x_t}{\Tilde{x}_t} \left(\pdv{\Tilde{x}_t}{x_0} + \dv{\Tilde{x}_t}{z_{t-1}} \dv{z_{t-1}}{\Tilde{z}_{t-1}}\pdv{\Tilde{z}_{t-1}}{x_0}\right) &= \\
        \dv{\phi_2(\Tilde{x}_t)}{\Tilde{x}_t} \bigg(\dv{h_2(g(z_{t-1}, x_0))}{g(z_{t-1}, x_0)}\pdv{g(z_{t-1}, x_0)}{x_0} + \left(I + \dv{h_2(g(z_{t-1}, x_0))}{g(z_{t-1}, x_0)}\pdv{g(z_{t-1}, x_0)}{z_{t-1}}\right)\\
        \qquad \cdot \dv{\phi_1(\Tilde{z}_{t-1})}{\Tilde{z}_{t-1}} \dv{h_1(g(x_{t-1}, x_0))}{g(x_{t-1}, x_0)}\pdv{g(x_{t-1}, x_0)}{x_0}\bigg)
    \end{align*}

\subsection{Local-Independence of $e$ in Stable Regimes}\label{app:independent-e}
\begin{proof}
    We wish to prove that in the stable regime found in Section \ref{sec:recall-outer}, the resulting fixed point is locally independent of $e$, the initial input to recall. Importantly, in both external and internal recall, this initial iterate \textit{does not affect $x_0$}, only later $x_t$. Thus, we have that $\dv{x_t}{e} = \dv{x_t}{x_{t-1}}\dv{x_{t-1}}{e}$. We unroll this back and see that \[
    \dv{x_T}{e} = \left(\prod_{i=2}^T \dv{x_t}{x_{t-1}}\right) \pdv{x_1}{e}
    \]

    As proven in Appendix \ref{app:input-gradient}, we know that when $\rho\left(\pdv{f(x^*, x_0)}{x^*}\right) < 1$ (i.e. when in the stable regime) the term $\prod_{i=2}^T \dv{x_t}{x_{t-1}}$ converges to 0 as $T \rightarrow \infty$. Thus the fixed point becomes \textit{locally-independent} of the initial iterate e while remaining \textit{dependent} on $x_0$ as proven in Appendix \ref{app:input-gradient}. 
\end{proof}

\section{Empirical Analyses}
\subsection{Experimental Setup}
\subsubsection{Data}\label{app:data}
We train on three datasets -- prefix-sums, sudoku, and chess -- each designed to test stability and generalization along a different axis.

\paragraph{Prefix-sums:} We obtained prefix-sum data from \cite{bansal_end--end_2022}. Each question inputs an N-bit binary sequence, and outputs a same-length sequence with each bit K the sum of all preceding input bits, mod 2. This task is also known as cumulative parity. We train and validate the model on 32-bit data (8000 and 2000 points, respectively) and test on 512-bit data (10000 points). We note that this is the \textit{only} problem whose difficulty is directly determined by its \textit{length}, rather than some alternative internal mechanism. 

\paragraph{Sudoku:} We obtained sudoku data from Sapient Intelligence's sudoku-extreme dataset (\citep{wang_hierarchical_2025}) that defines difficulty as the number of backtracks needed on the Tdoku solver (\citep{tdoku}) to solve the puzzle. To separate \enquote{easy} from \enquote{hard} examples, we calculated the 80th percentile of backtracks needed throughout the dataset (38 backtracks). We then trained on 40000 sudoku puzzles with fewer than the 80th percentile of backtracks, validated on 10000 of the same difficulty, and tested on 10000 puzzles with more than the 80th percentile of backtracks. Sudoku uses 10 input channels and 10 output channels (one each from 0-9).

\paragraph{Chess:} We train our models on data from the Lichess (\citep{lichess}) puzzles database, which provides FEN data for given chess board states as well as the corresponding best next move. We train our model on puzzles with ELO ratings less than 1600, and validate on the same; we test the model on harder puzzles with ratings between 1600-2000. Lichess puzzles remain the same size across training and testing ($8 \times 8$) and have 12 input channels (6 pieces, 2 players) and 2 output channels (1s on the from and to spots on the board). We normalize orientation so the model always sees \enquote{white to move}. We use 40000 training examples and 10000 of both validation and testing examples.

\begin{table}[h]
\centering
\caption{Dataset summary. For prefix-sums, train and test use different sequence lengths (32-bit and 512-bit, respectively), as length directly determines difficulty. Note that prefix-sums and chess both use two-logit classification.}
\label{tab:datasets}
\begin{tabular}{lccccc}
\toprule
\textbf{Dataset} & \textbf{Dimensions} & \textbf{In Channels} & \textbf{Out Channels} & \textbf{Train} & \textbf{Test} \\
\midrule
Prefix-Sums & $32 \to 512$ bits & 1 & 2 & 8{,}000 & 10{,}000 \\
Sudoku       & $9 \times 9$       & 10 & 10 & 40{,}000 & 10{,}000 \\
Chess        & $8 \times 8$       & 12 & 2  & 40{,}000 & 10{,}000 \\
\bottomrule
\end{tabular}
\end{table}

\subsubsection{Training Details}\label{app:training}

Across all datasets, we use a single-layer transformer with a hidden dimension of 256. The attention module uses 8 heads (head dimension 32), QK-normalization \citep{henry2020querykeynormalizationtransformers} for training stability, and multi-dimensional RoPE \citep{ropend} to provide consistent positional structure across both 1D (prefix-sums) and 2D (sudoku, chess) inputs. For sudoku and chess, we use full bidirectional attention. For prefix-sums, RoPE causes length generalization failures, so we instead use a sliding attention window of width 5. The MLP uses a hidden dimension of 1024 with GELU activations \citep{hendrycks2023gaussianerrorlinearunits}. We apply RMSNorm at all normalization sites. We sweep over three RMSNorm-based normalization configurations: pre-norm (pre), post-norm (post), peri-norm (peri) \citep{kim2025perilnrevisitingnormalizationlayer}, in addition to GRU outer normalization (gru). For models with GRU outer normalization, we replace the standard residual connection $x_{t+1} = x_t + h_\theta(x_t)$ with $x_{t+1} = \text{GRU}(x_t, h_\theta(x_t))$, where the recurrent state is $x_t$ and the update is $h_\theta(x_t)$. For recall models, we use linear recall, i.e., $g(x_t, x_0) = W_x x_t + W_0 x_0$.  Looped models with GRU had 1.39M parameters, while models without it had 0.98M. Fixed (non-looped) models used 15 layers and had 11.9M parameters without GRU and 17.8M parameters with it. Recall accounted for 0.13M parameters in models that used it. 

We note that while recall does add parameters, its benefits are not primarily attributable to parameter count. If they were, fixed-depth models -- which have an order of magnitude more parameters -- would dominate across tasks, yet they are outperformed by recall variants on sudoku and prefix-sums, and only match them on chess (the least naturally iterative of the three). \cite{bansal_end--end_2022} make a similar observation, noting that the gains from recall far exceed what additional depth or width provides in non-recall models.

We train with AdamW \citep{loshchilov2019decoupledweightdecayregularization} (weight decay 0.01) for 100 epochs with a batch size of 500. We apply an exponential warmup over the first 10 epochs, reaching $63\%$ of the target learning rate by epoch 9, followed by a constant learning rate through epoch 60, after which we apply a $10\times$ cooldown. Gradients are clipped at norm 1. 
 
Within each batch, we apply the progressive loss scheme of \citet{bansal_end--end_2022}. We sample $N \sim \mathcal{U}[0, T-1]$ and $K \sim \mathcal{U}[1, T-N]$, where $T = 30$ is the maximum number of iterations. The model runs $N$ iterations without gradients followed by $K$ iterations with gradients; loss is computed and backpropagated only over the final $K$ iterations. This is intended to discourage the model from learning iteration-specific shortcuts. At test-time, we unroll the model to 500 loops to ensure it has enough time to solve harder problems.

Experiments were conducted in PyTorch \citep{DBLP:journals/corr/abs-1912-01703}, and analyses used NumPy \citep{harris2020array} and Pandas \citep{mckinney-proc-scipy-2010}.
\begin{table}[h]
\centering
\caption{Hyperparameter summary for all experiments.}
\label{tab:hyperparams}
\begin{tabular}{ll}
\toprule
\textbf{Hyperparameter} & \textbf{Value} \\
\midrule
Hidden dimension        & 256 \\
Attention heads         & 8 (head dim 32) \\
MLP hidden dimension    & 1024 \\
Activation              & GELU \\
Normalization           & RMSNorm \\
Optimizer               & AdamW \\
Weight decay            & 0.01 \\
Batch size              & 500 \\
Epochs                  & 100 \\
LR warmup               & Epochs 0--9 (exponential) \\
LR cooldown             & $10\times$ reduction at epoch 60 \\
Gradient clip norm      & 1 \\
Max iterations ($T$)    & 30 \\
\bottomrule
\end{tabular}
\end{table}

\subsubsection{Compute Usage}\label{sec:compute}
Our final experiments used $120$ NVIDIA B200 GPU-hours and 35 MI325X GPU-hours. Additional experiments not included approximately doubled used compute. All models were trained on a single GPU. We trained all models with BF16 casting and sped up training via PyTorch's compile feature. 

\subsection{Anisotropy Analysis}\label{app:anisotropy}

\begin{table}[h!]
\centering
\footnotesize
\caption{Projected-point anisotropy; $n = 10000$ draws per $\sigma$. Parens after means: SE ($\hat\sigma/\sqrt{n}$)}
\label{tab:anisotropy-projection}
\setlength{\tabcolsep}{5pt}
\begin{tabular}{@{}l cc cc cc cc@{}}
\toprule
 & \multicolumn{4}{c}{log-range} & \multicolumn{4}{c}{balance} \\
\cmidrule(lr){2-5} \cmidrule(lr){6-9}
 & \multicolumn{2}{c}{mean} & \multicolumn{2}{c}{median} & \multicolumn{2}{c}{mean} & \multicolumn{2}{c}{median} \\
\cmidrule(lr){2-3} \cmidrule(lr){4-5} \cmidrule(lr){6-7} \cmidrule(lr){8-9}
$\sigma$ & int & ext & int & ext & int & ext & int & ext \\
\midrule
0.5 & 9.3359 (0.0831) & 1.1524 (0.0102) & 4.6835 & 0.8946 & 0.2250 (0.0030) & 0.4417 (0.0028) & 0.0092 & 0.4088 \\
1 & 9.6840 (0.0865) & 1.2031 (0.0098) & 4.6835 & 1.0320 & 0.2225 (0.0030) & 0.4144 (0.0027) & 0.0092 & 0.3563 \\
2 & 10.1580 (0.0888) & 1.4412 (0.0093) & 4.9472 & 1.4045 & 0.1803 (0.0026) & 0.3315 (0.0025) & 0.0071 & 0.2455 \\
4 & 11.1402 (0.0884) & 1.9594 (0.0106) & 18.1294 & 1.7632 & 0.0943 (0.0017) & 0.2197 (0.0020) & 0.0000 & 0.1715 \\
\bottomrule
\end{tabular}
\end{table}

Figure \ref{fig:stability_regions} in Section \ref{sec:recall-no-outer} showcases the stability regions for a simple one-layer model with overlapping eigenvectors. To understand how these regions affect the resulting chosen Jacobian eigenvalues, we conduct a simple projection-based experiment. 

In particular, we begin by sampling points from $\mathcal{N}(0, \sigma^2I_{2 \times 2})$ and then project them onto the internal and external stability regions. For each projected point, we calculate two metrics to analyze anisotropy: \begin{itemize}
    \item Log-range: $\abs{\log \frac{\abs{x} + \epsilon}{\abs{y} + \epsilon}}$, where more isotropic points have values closer to 0.
    \item Balance: $\frac{\min(\abs{x}, \abs{y}) + \epsilon}{\max(\abs{x}, \abs{y}) + \epsilon}$, where more isotropic points have values closer to 1.
\end{itemize}

Table \ref{tab:anisotropy-projection} showcases the result of this simulation. Across all choices of $\sigma$, internal recall is \textit{substantially} more anisotropic, having a median log-range at least three times as large as external recall and a median balance at least 4 times as small. 

We note that this experiment is purely a toy model designed for intuition, and degrades with misaligned eigenvectors since the spectral radius depends on a product of non-commuting matrices whose eigenvalues do not give the complete information. However, the result transfers empirically to real-world models: in \ref{app:jacobian-convergence}, we find that internal models with pre-norm never converge to a fixed point while external pre-norm can -- suggesting that indeed, the difference in Jacobian geometry at the fixed point for internal and external models do have a substantive effect on stability. 

\subsection{Jacobian spectral shrinkage analysis}\label{app:jacobian-convergence}
Outer normalization could stabilize iteration and improve convergence in two ways: directly, by shrinking the Jacobian at the fixed point, or indirectly, by pushing tokens towards a shared direction (the token clustering known to arise under post-norm \citep{karagodin2025normalizationattentiondynamics}) and collapsing the state. Here, we directly separate the two of these by conducting experiments analyzing both the spectral radii at convergence and the extent to which tokens are clustered at convergence.

In particular, we conduct power iteration on the full-loop Jacobian after convergence on Sudoku runs, to compare the spectral radius for models with and without outer normalization. We conduct these analyses on both single-block and double-block models, to confirm that all findings transfer beyond just one FFN-Attention combination. We find multiple important results. 

First, all converged models have a full-loop spectral radius less than 1, while many have a spectral radius of $W_x$ greater than one. This showcases the importance of considering the full-loop Jacobian rather than just the individual recall matrix. Second, removing the post-norm from converged models drastically increases the full-loop Jacobian spectral radius, indicating that post-norm does indeed act as a lever for controlling the spectral radius. At the same time, converged states are \textit{less} clustered than at initialization, with cosine similarity dropping by more than half for all one-block converged sudoku runs, with similar results for both normalized and unnormalized models. Two-block models show the same qualitative pattern, but in three of four post-norm configurations their fixed points are noticeably more clustered than two-block pre-norm ones (0.29--0.36, vs 0.20-0.21). These results suggest that while post-norm token clustering is not the mechanism behind its stability at this scale, post-norm at scale may indeed induce token clustering in looped models as in \citet{geiping_scaling_2025}.

\begin{table}[!htbp]
\centering
\caption{Full-loop Jacobian spectral radii for Sudoku models with one or two blocks, evaluated on five validation inputs per configuration. Full-loop spectral radii are averaged over inputs for which model iteration converged. The \enquote{no $\phi$} column tests the full-loop spectral radius with the post-norm removed, and is thus only valid on post-norm models.}
\label{tab:full-loop-jacobian}
\small
\setlength{\tabcolsep}{4pt}
\begin{tabular}{@{}llllccccc@{}}
\toprule
Norm & Recall & LR & Blocks & $\rho(W_x)$ & $\rho(M_{\text{full}})$ & $\rho(M_{\text{full}}, \text{no }\phi)$ & Conv. & Hard acc. \\
\midrule
post & external & 1e-3 & 1 & 1.61 & 0.96 & 5.59 & 4/5 & 16.2 \\
post & external & 1e-3 & 2 & 0.86 & 0.93 & 18.7 & 3/5 & 13.2 \\
post & external & 3e-3 & 1 & 3.90 & 0.96 & 39.1 & 4/5 & 2.4 \\
post & external & 3e-3 & 2 & 1.54 & 0.91 & 1{,}248 & 3/5 & 2.6 \\
post & internal & 1e-3 & 1 & 0.78 & 0.89 & 3.96 & 2/5 & 37.1 \\
post & internal & 1e-3 & 2 & 0.75 & 0.89 & 17.4 & 2/5 & 33.8 \\
post & internal & 3e-3 & 1 & 0.87 & 0.89 & 28.3 & 3/5 & 9.3 \\
post & internal & 3e-3 & 2 & 0.97 & 0.95 & 2{,}130 & 3/5 & 6.4 \\
\midrule
pre & external & 1e-3 & 1 & 0.84 & 0.96 & n/a & 5/5 & 21.0 \\
pre & external & 1e-3 & 2 & 0.65 & 0.96 & n/a & 3/5 & 17.2 \\
pre & external & 3e-3 & 1 & 1.06 & 0.97 & n/a & 4/5 & 0.0 \\
pre & external & 3e-3 & 2 & 0.63 & 0.93 & n/a & 4/5 & 0.0 \\
pre & internal & 1e-3 & 1 & 1.13 & --- & n/a & 0/5 & 1.6 \\
pre & internal & 1e-3 & 2 & 1.10 & --- & n/a & 0/5 & 5.5 \\
pre & internal & 3e-3 & 1 & 2.29 & --- & n/a & 0/5 & 0.3 \\
pre & internal & 3e-3 & 2 & 2.27 & --- & n/a & 0/5 & 0.2 \\
\bottomrule
\end{tabular}
\end{table}

\begin{table}[htbp]
\centering
\caption{Token cosine similarity at initialization and at converged fixed points for one and two block Sudoku models, averaged over converged inputs (of five) per configuration. Only configurations with converged inputs are shown.}
\label{tab:token-cosine}
\begin{tabular}{@{}llllcc@{}}
\toprule
Norm & Recall & LR & Blocks & Initial & Fixed point \\
\midrule
post & external & 1e-3 & 1 & 0.58 & 0.23 \\
post & external & 1e-3 & 2 & 0.61 & 0.29 \\
post & external & 3e-3 & 1 & 0.54 & 0.24 \\
post & external & 3e-3 & 2 & 0.50 & 0.30 \\
post & internal & 1e-3 & 1 & 0.56 & 0.23 \\
post & internal & 1e-3 & 2 & 0.58 & 0.36 \\
post & internal & 3e-3 & 1 & 0.55 & 0.17 \\
post & internal & 3e-3 & 2 & 0.53 & 0.12 \\
\midrule
pre & external & 1e-3 & 1 & 0.60 & 0.18 \\
pre & external & 1e-3 & 2 & 0.63 & 0.20 \\
pre & external & 3e-3 & 1 & 0.61 & 0.23 \\
pre & external & 3e-3 & 2 & 0.66 & 0.21 \\
\bottomrule
\end{tabular}
\end{table}

\subsection{Effect of learning rates on $\rho\left(\pdv{g(x_t, x_0)}{x_t}\right)$}\label{app:lr-rho}

\begin{figure}[h!]
    \centering
    \textbf{Effect of Learning Rate on Recall Spectral Radius}
    \includegraphics[width=\linewidth]{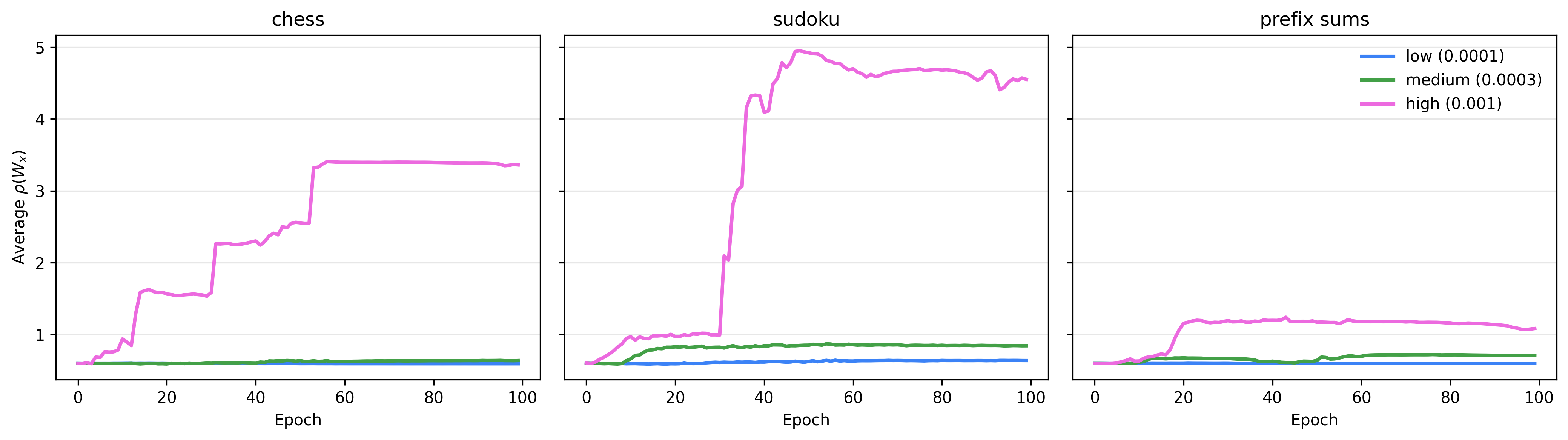}\\[1em]
    \includegraphics[width=\linewidth]{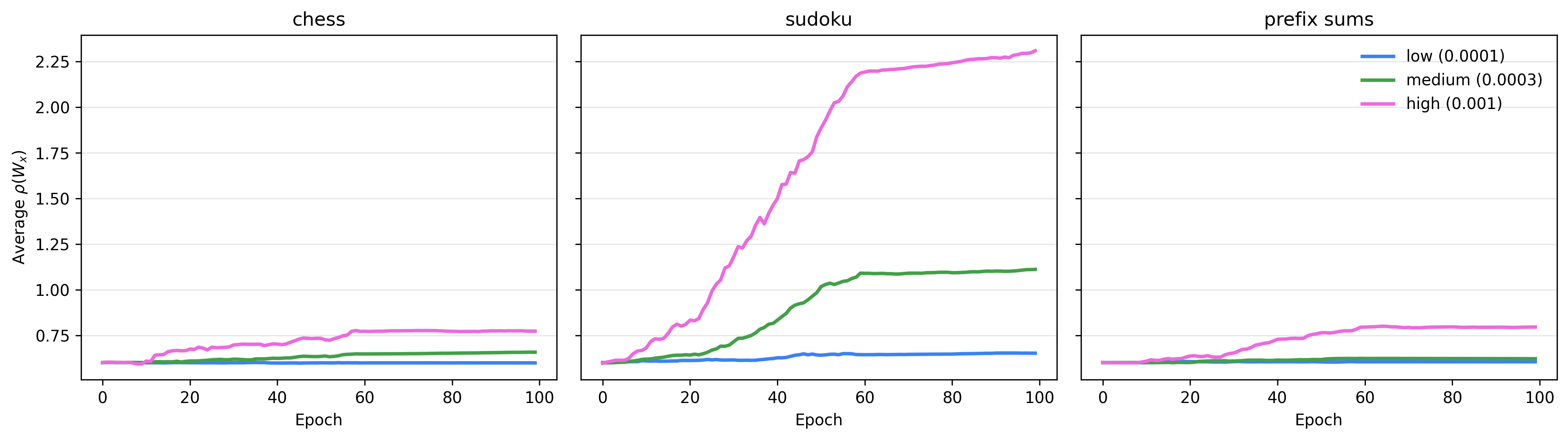}
    \caption{Effect of LR on non-outer-normalized internal and external recall Jacobian spectral radius. Top is external recall, bottom is internal recall. Chess and prefix sums use ($0.0001, 0.0003, 0.001$); sudoku uses ($0.0003, 0.001, 0.003)$}
    \label{fig:lr-spectral}
\end{figure}

In Section \ref{sec:recall-no-outer}, we discuss the effects of learning rate on the spectral radius of $\pdv{g(x^*, x_0)}{x^*}$. Figure \ref{fig:lr-spectral} makes this explicit: across all problems and recall choices (averaged over pre/peri norm), higher LR corresponds to larger spectral radius for $\pdv{g(x^*, x_0)}{x^*}$. We use $g(x_t, x_0) = W_x x_t + W_0 x_0$, so $\pdv{g(x^*, x_0)}{x^*} = W_x$. 

\subsection{Effect of $\rho\left(\pdv{g(x_t, x_0)}{x_t}\right)$ on External Recall Performance}

In section \ref{sec:recall-no-outer}, we discuss how, without outer normalization, 
external recall can have difficulty creating locally-attractive fixed points when $\pdv{g(x^*, x_0)}{x^*}$ becomes large, as the other eigenvalues have to compensate by shifting towards a complex unit ball centered around $-1$. Here, we provide empirical evidence for this fact: Figure \ref{fig:spectral} compares final (normalized by task) accuracy for external recall, unnormalized models across tasks as a function of $\rho\left(\pdv{g(x^*, x_0)}{x^*}\right) = \rho\left(\pdv{g(x_t, x_0)}{x_t}\right) = \rho(W_x)$ for linear recall mechanism $g(x_t, x_0) = W_x x_t + W_0 x_0$. Indeed, we see that external recall performs substantially worse on both val and hard data as the spectral radius of $W_x$ grows beyond 1, with no run with $\rho(W_x) > 1.1$ achieving non-zero accuracy on hard data.

\subsection{Seed-Averaged Results for All Models}\label{app:full-results}
In Tables \ref{tab:val} and \ref{tab:hard}, we provide the validation and hard accuracy for all (problem, norm, lr, recall) tuples, averaged over 10 seeds. We also provide standard deviations (SD) over all seeds for uncertainty analysis. We bold the best result in each group (i.e. validation sudoku, or hard chess) and underline the second best.

Each value is \enquote{total} accuracy; i.e. not bitwise accuracy, but rather the percent of inputs whose answer is completely correct. Sudoku consistently required higher learning rates to reach equivalent accuracy; as such, we report higher LR here and report the full LR sweep in Appendix \ref{app:sudoku}. 

We note that for fixed-length models on prefix-sums, using sliding-window attention does not provide a large enough receptive field over 15 layers, ensuring 0\% accuracy. For fair comparisons, we remove this model from our analysis on hard data, as using a full-length model would unfairly correlate failures of positional embedding mechanisms with that of generalization ability.

\begin{table}[h!]
\caption{Accuracy on validation (easy distribution, but not trained on) data. SD is given as $\pm$; bolding is best, underlined is second best.}
\label{tab:val}
\centering
\setlength{\tabcolsep}{2.7pt}
\footnotesize
\begin{tabular}{@{}ll ccccccccc@{}}
\toprule
\multirow{2}{*}{Norm} & \multirow{2}{*}{Recall} & \multicolumn{3}{c}{Chess} & \multicolumn{3}{c}{Sudoku} & \multicolumn{3}{c}{Prefix\_sums} \\
\cmidrule(lr){3-5}\cmidrule(lr){6-8}\cmidrule(lr){9-11}
 &  & {\footnotesize\num{1e-4}} & {\footnotesize\num{3e-4}} & {\footnotesize\num{1e-3}} & {\footnotesize\num{3e-4}} & {\footnotesize\num{1e-3}} & {\footnotesize\num{3e-3}} & {\footnotesize\num{1e-4}} & {\footnotesize\num{3e-4}} & {\footnotesize\num{1e-3}} \\
\midrule
\multirow{4}{*}{gru} & internal & 57.1{\scriptsize$\pm$6.1} & 71.8{\scriptsize$\pm$1.7} & 67.4{\scriptsize$\pm$6.3} & 28.9{\scriptsize$\pm$10.3} & 35.2{\scriptsize$\pm$9.0} & 25.3{\scriptsize$\pm$5.1} & \textbf{100.0{\scriptsize$\pm$0.0}} & 84.9{\scriptsize$\pm$32.3} & 86.3{\scriptsize$\pm$30.5} \\
 & external & 70.3{\scriptsize$\pm$0.9} & 75.0{\scriptsize$\pm$1.1} & \textbf{78.3{\scriptsize$\pm$1.0}} & 35.0{\scriptsize$\pm$1.2} & 40.1{\scriptsize$\pm$2.3} & 0.0{\scriptsize$\pm$0.0} & \textbf{100.0{\scriptsize$\pm$0.0}} & 88.2{\scriptsize$\pm$31.0} & 29.1{\scriptsize$\pm$46.8} \\
 & none & 37.3{\scriptsize$\pm$14.2} & 43.9{\scriptsize$\pm$5.2} & 21.1{\scriptsize$\pm$15.0} & 6.7{\scriptsize$\pm$14.0} & 4.9{\scriptsize$\pm$8.8} & 19.3{\scriptsize$\pm$8.1} & 95.0{\scriptsize$\pm$7.8} & 55.5{\scriptsize$\pm$48.0} & 8.2{\scriptsize$\pm$17.1} \\
 & fixed & 72.4{\scriptsize$\pm$0.6} & 76.2{\scriptsize$\pm$0.5} & 77.9{\scriptsize$\pm$0.6} & 2.8{\scriptsize$\pm$0.1} & 3.0{\scriptsize$\pm$0.1} & 0.8{\scriptsize$\pm$1.3} & 2.8{\scriptsize$\pm$5.6} & 0.1{\scriptsize$\pm$0.3} & 0.0{\scriptsize$\pm$0.0} \\
\specialrule{0.2pt}{0.15em}{0.15em}
\multirow{4}{*}{peri} & internal & 40.9{\scriptsize$\pm$3.5} & 55.1{\scriptsize$\pm$1.7} & 58.7{\scriptsize$\pm$2.4} & 29.9{\scriptsize$\pm$2.1} & 42.6{\scriptsize$\pm$3.6} & 26.0{\scriptsize$\pm$7.7} & 99.3{\scriptsize$\pm$0.8} & 98.9{\scriptsize$\pm$1.7} & 76.3{\scriptsize$\pm$31.2} \\
 & external & 66.1{\scriptsize$\pm$2.5} & 72.7{\scriptsize$\pm$1.5} & 45.7{\scriptsize$\pm$39.4} & 45.8{\scriptsize$\pm$2.2} & 51.7{\scriptsize$\pm$2.0} & 37.6{\scriptsize$\pm$13.4} & \textbf{100.0{\scriptsize$\pm$0.0}} & 90.0{\scriptsize$\pm$31.6} & 10.0{\scriptsize$\pm$31.6} \\
 & none & 28.5{\scriptsize$\pm$3.6} & 27.3{\scriptsize$\pm$4.7} & 16.2{\scriptsize$\pm$5.3} & 25.1{\scriptsize$\pm$3.5} & 28.0{\scriptsize$\pm$3.8} & 2.4{\scriptsize$\pm$7.2} & 34.7{\scriptsize$\pm$28.2} & 0.0{\scriptsize$\pm$0.0} & 0.0{\scriptsize$\pm$0.0} \\
 & fixed & 71.4{\scriptsize$\pm$0.7} & 74.3{\scriptsize$\pm$0.8} & 76.7{\scriptsize$\pm$0.6} & 2.9{\scriptsize$\pm$0.1} & 3.1{\scriptsize$\pm$0.1} & 3.2{\scriptsize$\pm$0.2} & 26.5{\scriptsize$\pm$26.2} & 46.1{\scriptsize$\pm$25.5} & 5.0{\scriptsize$\pm$9.0} \\
\specialrule{0.2pt}{0.15em}{0.15em}
\multirow{4}{*}{post} & internal & 28.3{\scriptsize$\pm$6.2} & 70.6{\scriptsize$\pm$1.4} & 68.6{\scriptsize$\pm$9.0} & 40.5{\scriptsize$\pm$1.7} & \textbf{57.8{\scriptsize$\pm$1.5}} & 35.5{\scriptsize$\pm$13.1} & \underline{99.9{\scriptsize$\pm$0.3}} & 99.0{\scriptsize$\pm$1.3} & 19.2{\scriptsize$\pm$39.4} \\
 & external & 62.1{\scriptsize$\pm$1.8} & 72.5{\scriptsize$\pm$1.6} & 73.3{\scriptsize$\pm$3.5} & 32.0{\scriptsize$\pm$1.6} & 47.3{\scriptsize$\pm$2.2} & 30.0{\scriptsize$\pm$4.3} & 99.7{\scriptsize$\pm$0.3} & 99.6{\scriptsize$\pm$0.8} & 29.1{\scriptsize$\pm$46.8} \\
 & none & 33.4{\scriptsize$\pm$13.2} & 21.6{\scriptsize$\pm$6.2} & 0.3{\scriptsize$\pm$0.9} & 37.5{\scriptsize$\pm$4.1} & 2.7{\scriptsize$\pm$7.7} & 0.0{\scriptsize$\pm$0.0} & 0.0{\scriptsize$\pm$0.0} & 0.0{\scriptsize$\pm$0.0} & 0.0{\scriptsize$\pm$0.0} \\
 & fixed & 70.4{\scriptsize$\pm$0.4} & 75.0{\scriptsize$\pm$0.7} & 63.8{\scriptsize$\pm$33.6} & 2.7{\scriptsize$\pm$0.0} & 3.3{\scriptsize$\pm$0.1} & 0.0{\scriptsize$\pm$0.0} & 23.6{\scriptsize$\pm$10.5} & 15.3{\scriptsize$\pm$17.1} & 0.0{\scriptsize$\pm$0.0} \\
\specialrule{0.2pt}{0.15em}{0.15em}
\multirow{4}{*}{pre} & internal & 38.7{\scriptsize$\pm$3.3} & 47.7{\scriptsize$\pm$7.4} & 45.5{\scriptsize$\pm$3.6} & 11.2{\scriptsize$\pm$7.3} & 27.1{\scriptsize$\pm$4.9} & 20.8{\scriptsize$\pm$3.4} & 98.7{\scriptsize$\pm$1.1} & 96.8{\scriptsize$\pm$7.8} & 79.8{\scriptsize$\pm$33.2} \\
 & external & 69.6{\scriptsize$\pm$0.9} & 67.4{\scriptsize$\pm$23.7} & 0.0{\scriptsize$\pm$0.0} & 43.9{\scriptsize$\pm$1.7} & \underline{52.2{\scriptsize$\pm$2.3}} & 6.4{\scriptsize$\pm$13.6} & \textbf{100.0{\scriptsize$\pm$0.0}} & 79.7{\scriptsize$\pm$42.0} & 0.0{\scriptsize$\pm$0.0} \\
 & none & 31.3{\scriptsize$\pm$4.7} & 27.0{\scriptsize$\pm$5.6} & 14.2{\scriptsize$\pm$7.4} & 5.4{\scriptsize$\pm$5.7} & 14.2{\scriptsize$\pm$11.0} & 13.8{\scriptsize$\pm$7.1} & 80.4{\scriptsize$\pm$32.1} & 9.3{\scriptsize$\pm$27.4} & 0.0{\scriptsize$\pm$0.0} \\
 & fixed & 71.2{\scriptsize$\pm$0.5} & 75.1{\scriptsize$\pm$0.6} & \underline{78.0{\scriptsize$\pm$0.5}} & 2.8{\scriptsize$\pm$0.0} & 3.1{\scriptsize$\pm$0.1} & 3.9{\scriptsize$\pm$0.9} & 42.1{\scriptsize$\pm$15.4} & 60.3{\scriptsize$\pm$24.9} & 15.3{\scriptsize$\pm$17.4} \\
\bottomrule
\end{tabular}
\end{table}

\begin{table}[h!]
\caption{Accuracy on hard data. SD is given as $\pm$; bolding is best, underlined is second best.}
\label{tab:hard}
\centering
\setlength{\tabcolsep}{2.7pt}
\footnotesize
\begin{tabular}{@{}ll ccccccccc@{}}
\toprule
\multirow{2}{*}{Norm} & \multirow{2}{*}{Recall} & \multicolumn{3}{c}{Chess} & \multicolumn{3}{c}{Sudoku} & \multicolumn{3}{c}{Prefix\_sums} \\
\cmidrule(lr){3-5}\cmidrule(lr){6-8}\cmidrule(lr){9-11}
 &  & {\footnotesize\num{1e-4}} & {\footnotesize\num{3e-4}} & {\footnotesize\num{1e-3}} & {\footnotesize\num{3e-4}} & {\footnotesize\num{1e-3}} & {\footnotesize\num{3e-3}} & {\footnotesize\num{1e-4}} & {\footnotesize\num{3e-4}} & {\footnotesize\num{1e-3}} \\
\midrule
\multirow{4}{*}{gru} & internal & 26.7{\scriptsize$\pm$3.4} & 36.2{\scriptsize$\pm$1.3} & 32.7{\scriptsize$\pm$5.2} & 3.5{\scriptsize$\pm$3.3} & 6.3{\scriptsize$\pm$4.0} & 0.8{\scriptsize$\pm$0.7} & 92.9{\scriptsize$\pm$15.5} & 78.5{\scriptsize$\pm$41.4} & 72.1{\scriptsize$\pm$31.5} \\
 & external & 34.9{\scriptsize$\pm$0.8} & 38.4{\scriptsize$\pm$0.9} & \textbf{41.3{\scriptsize$\pm$1.2}} & 5.2{\scriptsize$\pm$1.0} & 8.5{\scriptsize$\pm$1.8} & 0.0{\scriptsize$\pm$0.0} & \textbf{99.9{\scriptsize$\pm$0.2}} & 84.4{\scriptsize$\pm$31.3} & 28.9{\scriptsize$\pm$46.6} \\
 & none & 14.5{\scriptsize$\pm$5.8} & 17.2{\scriptsize$\pm$2.7} & 7.9{\scriptsize$\pm$5.7} & 1.2{\scriptsize$\pm$2.6} & 0.4{\scriptsize$\pm$0.9} & 0.4{\scriptsize$\pm$0.7} & 0.0{\scriptsize$\pm$0.0} & 0.0{\scriptsize$\pm$0.0} & 0.0{\scriptsize$\pm$0.0} \\
 & fixed & 35.5{\scriptsize$\pm$0.7} & 39.3{\scriptsize$\pm$0.6} & \underline{40.5{\scriptsize$\pm$0.8}} & 0.0{\scriptsize$\pm$0.0} & 0.0{\scriptsize$\pm$0.0} & 0.0{\scriptsize$\pm$0.0} & --- & --- & --- \\
\specialrule{0.2pt}{0.15em}{0.15em}
\multirow{4}{*}{peri} & internal & 15.8{\scriptsize$\pm$2.0} & 24.0{\scriptsize$\pm$1.1} & 25.6{\scriptsize$\pm$2.0} & 3.4{\scriptsize$\pm$1.3} & 14.4{\scriptsize$\pm$4.3} & 2.2{\scriptsize$\pm$2.0} & 0.0{\scriptsize$\pm$0.0} & 0.0{\scriptsize$\pm$0.0} & 0.1{\scriptsize$\pm$0.1} \\
 & external & 31.1{\scriptsize$\pm$2.1} & 36.8{\scriptsize$\pm$1.5} & 23.8{\scriptsize$\pm$20.5} & 12.8{\scriptsize$\pm$2.2} & 20.3{\scriptsize$\pm$3.6} & 8.2{\scriptsize$\pm$3.2} & 81.5{\scriptsize$\pm$30.1} & 90.0{\scriptsize$\pm$31.6} & 10.0{\scriptsize$\pm$31.6} \\
 & none & 10.4{\scriptsize$\pm$1.7} & 10.2{\scriptsize$\pm$1.9} & 5.8{\scriptsize$\pm$1.7} & 2.8{\scriptsize$\pm$1.2} & 3.1{\scriptsize$\pm$2.0} & 0.0{\scriptsize$\pm$0.0} & 0.0{\scriptsize$\pm$0.0} & 0.0{\scriptsize$\pm$0.0} & 0.0{\scriptsize$\pm$0.0} \\
 & fixed & 34.6{\scriptsize$\pm$0.5} & 37.1{\scriptsize$\pm$0.7} & 39.0{\scriptsize$\pm$0.6} & 0.0{\scriptsize$\pm$0.0} & 0.0{\scriptsize$\pm$0.0} & 0.0{\scriptsize$\pm$0.0} & --- & --- & --- \\
\specialrule{0.2pt}{0.15em}{0.15em}
\multirow{4}{*}{post} & internal & 13.1{\scriptsize$\pm$2.7} & 35.5{\scriptsize$\pm$1.5} & 33.1{\scriptsize$\pm$7.0} & 9.1{\scriptsize$\pm$1.5} & \textbf{34.7{\scriptsize$\pm$2.8}} & 7.4{\scriptsize$\pm$4.0} & 73.7{\scriptsize$\pm$37.7} & 98.8{\scriptsize$\pm$1.5} & 14.4{\scriptsize$\pm$32.0} \\
 & external & 28.0{\scriptsize$\pm$1.5} & 36.3{\scriptsize$\pm$1.4} & 37.0{\scriptsize$\pm$3.1} & 3.4{\scriptsize$\pm$0.9} & 16.8{\scriptsize$\pm$3.9} & 2.6{\scriptsize$\pm$1.4} & 80.1{\scriptsize$\pm$33.3} & \underline{99.5{\scriptsize$\pm$0.9}} & 19.5{\scriptsize$\pm$40.3} \\
 & none & 12.7{\scriptsize$\pm$4.8} & 8.5{\scriptsize$\pm$2.9} & 0.1{\scriptsize$\pm$0.4} & 8.5{\scriptsize$\pm$1.6} & 0.1{\scriptsize$\pm$0.2} & 0.0{\scriptsize$\pm$0.0} & 0.0{\scriptsize$\pm$0.0} & 0.0{\scriptsize$\pm$0.0} & 0.0{\scriptsize$\pm$0.0} \\
 & fixed & 34.0{\scriptsize$\pm$0.6} & 37.5{\scriptsize$\pm$0.6} & 33.7{\scriptsize$\pm$17.8} & 0.0{\scriptsize$\pm$0.0} & 0.0{\scriptsize$\pm$0.0} & 0.0{\scriptsize$\pm$0.0} & --- & --- & --- \\
\specialrule{0.2pt}{0.15em}{0.15em}
\multirow{4}{*}{pre} & internal & 14.9{\scriptsize$\pm$1.6} & 19.6{\scriptsize$\pm$3.7} & 18.1{\scriptsize$\pm$1.9} & 0.0{\scriptsize$\pm$0.0} & 2.1{\scriptsize$\pm$1.7} & 0.1{\scriptsize$\pm$0.1} & 0.0{\scriptsize$\pm$0.0} & 0.0{\scriptsize$\pm$0.0} & 0.1{\scriptsize$\pm$0.3} \\
 & external & 34.4{\scriptsize$\pm$0.5} & 34.7{\scriptsize$\pm$12.2} & 0.0{\scriptsize$\pm$0.0} & 10.9{\scriptsize$\pm$1.7} & \underline{21.3{\scriptsize$\pm$4.0}} & 0.6{\scriptsize$\pm$1.4} & 95.6{\scriptsize$\pm$13.0} & 69.7{\scriptsize$\pm$46.6} & 0.0{\scriptsize$\pm$0.0} \\
 & none & 11.6{\scriptsize$\pm$2.0} & 10.1{\scriptsize$\pm$2.0} & 5.0{\scriptsize$\pm$2.0} & 0.0{\scriptsize$\pm$0.0} & 0.5{\scriptsize$\pm$1.5} & 0.1{\scriptsize$\pm$0.2} & 0.0{\scriptsize$\pm$0.0} & 0.0{\scriptsize$\pm$0.0} & 0.0{\scriptsize$\pm$0.0} \\
 & fixed & 34.8{\scriptsize$\pm$0.4} & 37.8{\scriptsize$\pm$0.6} & \underline{40.5{\scriptsize$\pm$0.4}} & 0.0{\scriptsize$\pm$0.0} & 0.0{\scriptsize$\pm$0.0} & 0.2{\scriptsize$\pm$0.2} & --- & --- & --- \\
\bottomrule
\end{tabular}
\end{table}

\subsection{Sudoku Results}\label{app:sudoku}
For transparency, we report our results on sudoku across all four learning rates in Table \ref{tab:sudoku}. We found throughout our analysis that sudoku consistently needed larger learning rates regardless of model architecture. 

\begin{table}[h!]
\caption{Validation and hard accuracy for sudoku across all four learning rates. SD is given as $\pm$; bolding is best, underlined is second best.}
\label{tab:sudoku}
\centering
\setlength{\tabcolsep}{4.5pt}
\footnotesize
\begin{tabular}{@{}ll cccccccc@{}}
\toprule
\multirow{2}{*}{Norm} & \multirow{2}{*}{Recall} & \multicolumn{4}{c}{Validation Accuracy} & \multicolumn{4}{c}{Hard Accuracy} \\
\cmidrule(lr){3-6}\cmidrule(lr){7-10}
 &  & {\footnotesize\num{1e-4}} & {\footnotesize\num{3e-4}} & {\footnotesize\num{1e-3}} & {\footnotesize\num{3e-3}} & {\footnotesize\num{1e-4}} & {\footnotesize\num{3e-4}} & {\footnotesize\num{1e-3}} & {\footnotesize\num{3e-3}} \\
\midrule
\multirow{4}{*}{gru} & internal & 25.5{\scriptsize$\pm$1.0} & 28.9{\scriptsize$\pm$10.3} & 35.2{\scriptsize$\pm$9.0} & 25.3{\scriptsize$\pm$5.1} & 0.3{\scriptsize$\pm$0.2} & 3.5{\scriptsize$\pm$3.3} & 6.3{\scriptsize$\pm$4.0} & 0.8{\scriptsize$\pm$0.7} \\
 & external & 24.6{\scriptsize$\pm$1.2} & 35.0{\scriptsize$\pm$1.2} & 40.1{\scriptsize$\pm$2.3} & 0.0{\scriptsize$\pm$0.0} & 1.0{\scriptsize$\pm$0.8} & 5.2{\scriptsize$\pm$1.0} & 8.5{\scriptsize$\pm$1.8} & 0.0{\scriptsize$\pm$0.0} \\
 & none & 8.8{\scriptsize$\pm$5.4} & 6.7{\scriptsize$\pm$14.0} & 4.9{\scriptsize$\pm$8.8} & 19.3{\scriptsize$\pm$8.1} & 0.0{\scriptsize$\pm$0.0} & 1.2{\scriptsize$\pm$2.6} & 0.4{\scriptsize$\pm$0.9} & 0.4{\scriptsize$\pm$0.7} \\
 & fixed & 2.5{\scriptsize$\pm$0.1} & 2.8{\scriptsize$\pm$0.1} & 3.0{\scriptsize$\pm$0.1} & 0.8{\scriptsize$\pm$1.3} & 0.0{\scriptsize$\pm$0.0} & 0.0{\scriptsize$\pm$0.0} & 0.0{\scriptsize$\pm$0.0} & 0.0{\scriptsize$\pm$0.0} \\
\specialrule{0.2pt}{0.15em}{0.15em}
\multirow{4}{*}{peri} & internal & 2.2{\scriptsize$\pm$0.9} & 29.9{\scriptsize$\pm$2.1} & 42.6{\scriptsize$\pm$3.6} & 26.0{\scriptsize$\pm$7.7} & 0.0{\scriptsize$\pm$0.0} & 3.4{\scriptsize$\pm$1.3} & 14.4{\scriptsize$\pm$4.3} & 2.2{\scriptsize$\pm$2.0} \\
 & external & 25.0{\scriptsize$\pm$0.9} & 45.8{\scriptsize$\pm$2.2} & 51.7{\scriptsize$\pm$2.0} & 37.6{\scriptsize$\pm$13.4} & 1.0{\scriptsize$\pm$0.6} & 12.8{\scriptsize$\pm$2.2} & 20.3{\scriptsize$\pm$3.6} & 8.2{\scriptsize$\pm$3.2} \\
 & none & 0.7{\scriptsize$\pm$0.8} & 25.1{\scriptsize$\pm$3.5} & 28.0{\scriptsize$\pm$3.8} & 2.4{\scriptsize$\pm$7.2} & 0.0{\scriptsize$\pm$0.0} & 2.8{\scriptsize$\pm$1.2} & 3.1{\scriptsize$\pm$2.0} & 0.0{\scriptsize$\pm$0.0} \\
 & fixed & 2.5{\scriptsize$\pm$0.1} & 2.9{\scriptsize$\pm$0.1} & 3.1{\scriptsize$\pm$0.1} & 3.2{\scriptsize$\pm$0.2} & 0.0{\scriptsize$\pm$0.0} & 0.0{\scriptsize$\pm$0.0} & 0.0{\scriptsize$\pm$0.0} & 0.0{\scriptsize$\pm$0.0} \\
\specialrule{0.2pt}{0.15em}{0.15em}
\multirow{4}{*}{post} & internal & 14.1{\scriptsize$\pm$7.9} & 40.5{\scriptsize$\pm$1.7} & \textbf{57.8{\scriptsize$\pm$1.5}} & 35.5{\scriptsize$\pm$13.1} & 1.1{\scriptsize$\pm$0.8} & 9.1{\scriptsize$\pm$1.5} & \textbf{34.7{\scriptsize$\pm$2.8}} & 7.4{\scriptsize$\pm$4.0} \\
 & external & 15.3{\scriptsize$\pm$3.7} & 32.0{\scriptsize$\pm$1.6} & 47.3{\scriptsize$\pm$2.2} & 30.0{\scriptsize$\pm$4.3} & 0.7{\scriptsize$\pm$0.5} & 3.4{\scriptsize$\pm$0.9} & 16.8{\scriptsize$\pm$3.9} & 2.6{\scriptsize$\pm$1.4} \\
 & none & 9.0{\scriptsize$\pm$6.9} & 37.5{\scriptsize$\pm$4.1} & 2.7{\scriptsize$\pm$7.7} & 0.0{\scriptsize$\pm$0.0} & 0.5{\scriptsize$\pm$0.6} & 8.5{\scriptsize$\pm$1.6} & 0.1{\scriptsize$\pm$0.2} & 0.0{\scriptsize$\pm$0.0} \\
 & fixed & 2.1{\scriptsize$\pm$0.1} & 2.7{\scriptsize$\pm$0.0} & 3.3{\scriptsize$\pm$0.1} & 0.0{\scriptsize$\pm$0.0} & 0.0{\scriptsize$\pm$0.0} & 0.0{\scriptsize$\pm$0.0} & 0.0{\scriptsize$\pm$0.0} & 0.0{\scriptsize$\pm$0.0} \\
\specialrule{0.2pt}{0.15em}{0.15em}
\multirow{4}{*}{pre} & internal & 1.7{\scriptsize$\pm$0.5} & 11.2{\scriptsize$\pm$7.3} & 27.1{\scriptsize$\pm$4.9} & 20.8{\scriptsize$\pm$3.4} & 0.0{\scriptsize$\pm$0.0} & 0.0{\scriptsize$\pm$0.0} & 2.1{\scriptsize$\pm$1.7} & 0.1{\scriptsize$\pm$0.1} \\
 & external & 24.7{\scriptsize$\pm$1.2} & 43.9{\scriptsize$\pm$1.7} & \underline{52.2{\scriptsize$\pm$2.3}} & 6.4{\scriptsize$\pm$13.6} & 1.0{\scriptsize$\pm$0.7} & 10.9{\scriptsize$\pm$1.7} & \underline{21.3{\scriptsize$\pm$4.0}} & 0.6{\scriptsize$\pm$1.4} \\
 & none & 0.4{\scriptsize$\pm$0.4} & 5.4{\scriptsize$\pm$5.7} & 14.2{\scriptsize$\pm$11.0} & 13.8{\scriptsize$\pm$7.1} & 0.0{\scriptsize$\pm$0.0} & 0.0{\scriptsize$\pm$0.0} & 0.5{\scriptsize$\pm$1.5} & 0.1{\scriptsize$\pm$0.2} \\
 & fixed & 2.4{\scriptsize$\pm$0.1} & 2.8{\scriptsize$\pm$0.0} & 3.1{\scriptsize$\pm$0.1} & 3.9{\scriptsize$\pm$0.9} & 0.0{\scriptsize$\pm$0.0} & 0.0{\scriptsize$\pm$0.0} & 0.0{\scriptsize$\pm$0.0} & 0.2{\scriptsize$\pm$0.2} \\
\bottomrule
\end{tabular}
\end{table}
\end{document}